\documentclass{article}

\usepackage{graphicx}
\usepackage[utf8]{inputenc} 
\usepackage[T1]{fontenc}    
\usepackage[colorlinks]{hyperref}       
\usepackage{url}            
\usepackage{booktabs}       
\usepackage{amsfonts}       
\usepackage{nicefrac}       
\usepackage{microtype}      
\usepackage[parfill]{parskip}
\usepackage{algorithm,algorithmic}
\usepackage{amsmath,amsthm,amssymb,bbm}
\usepackage{mathtools}
\usepackage{cases}
\usepackage{comment}
\usepackage{caption}
\usepackage{subcaption}
\usepackage{algorithm,algorithmic}
\usepackage{color}
\usepackage{appendix}
\usepackage{xspace}
\usepackage{enumitem}
\usepackage[english]{babel}
\usepackage{authblk}

\usepackage{cleveref}
\crefformat{equation}{(#2#1#3)}
\crefrangeformat{equation}{(#3#1#4) to~(#5#2#6)}
\crefname{equation}{}{}
\Crefname{equation}{}{}

\crefname{definition}{\textbf{definition}}{definitions}
\Crefname{definition}{Definition}{Definitions}
\crefname{assumption}{\textbf{assumption}}{assumptions}
\Crefname{assumption}{Assumption}{Assumptions}
\definecolor{maroon}{RGB}{192,80,77}

\newcommand{\explain}[2]{\underset{\mathclap{\overset{\uparrow}{#2}}}{#1}}

\newcommand\independent{\protect\mathpalette{\protect\independenT}{\perp}}
\def\independenT#1#2{\mathrel{\rlap{$#1#2$}\mkern2mu{#1#2}}}

\newtheorem{theorem}{Theorem}
\newtheorem{lemma}[theorem]{Lemma}

\newtheorem{proposition}[theorem]{Proposition}

	\AtEndDocument{\refstepcounter{theorem}\label{finalthm}}

\newcommand{\vct}{\boldsymbol }
\newcommand{\mat}{\mathbf}
\newcommand{\red}[1]{\textcolor{red}{#1}}

\newcommand{\R}{\mathbb{R}}

\usepackage{mdwlist}

\def\E{\mathbb{E}}
\def\P{\mathbb{P}}

\def\diag{\mathrm{diag}}

\def\R{\mathbb{R}}

\def\cF{\mathcal{F}}

\def\cH{\mathcal{H}}

\def\cS{\mathcal{S}}

\def\cX{\mathcal{X}}
\def\cY{\mathcal{Y}}

\newcommand{\algolong}{Black Box Shift Estimation}
\newcommand{\algo}{\textsc{BBSE}}

\usepackage[accepted]{icml2018}

\icmltitlerunning{Detecting and Correcting for Label Shift with Black Box Predictors}

\begin{document}

\twocolumn[
\icmltitle{Detecting and Correcting for Label Shift with Black Box Predictors}

\icmlsetsymbol{equal}{*}

\begin{icmlauthorlist}
\icmlauthor{Zachary C. Lipton}{equal,cmu,amazon}
\icmlauthor{Yu-Xiang Wang}{equal,amazon,ucsb}
\icmlauthor{Alexander J. Smola}{amazon}
\end{icmlauthorlist}

\icmlaffiliation{cmu}{Carnegie Mellon University, Pittsburgh, PA}
\icmlaffiliation{amazon}{Amazon AI, Palo Alto, CA}
\icmlaffiliation{ucsb}{UC Santa Barbara, CA}

\icmlcorrespondingauthor{Zachary C. Lipton}{zlipton@cmu.edu }
\icmlcorrespondingauthor{Yu-Xiang Wang}{yuxiangw@amazon.com}
\icmlcorrespondingauthor{Alexander J. Smola}{smola@amazon.com}


\vskip 0.3in
]

\printAffiliationsAndNotice{\icmlEqualContribution}

\begin{abstract}
Faced with  distribution shift between training and test set, 
we wish to \emph{detect} and \emph{quantify} the shift, and to
\emph{correct} our classifiers without test set labels. 
Motivated by medical diagnosis,
where diseases (targets),
cause symptoms (observations),
we focus on \emph{label shift},
where the label marginal $p(y)$
changes but the conditional $p(\vct x| y)$ does not.
We propose \algolong{} (\algo{})
to estimate the test distribution $p(y)$.
\algo{} exploits arbitrary black box predictors 
to reduce dimensionality prior to shift correction.
While better predictors give tighter estimates,
\algo{} works even when predictors are biased, inaccurate, or uncalibrated,
so long as their confusion matrices are invertible. 
We prove 
\algo{}'s consistency, 
bound its error, 
and introduce a statistical test
that uses \algo{} to detect shift.
We also leverage \algo{}
to correct classifiers.
Experiments demonstrate accurate estimates and improved prediction, 
even on high-dimensional datasets of natural images.

\end{abstract}

\section{Introduction}
\label{sec:intro}

Assume that in August
we train a pneumonia predictor.
Our features consist
of chest X-rays administered 
in the previous year (distribution $P$)
and the labels binary indicators 
of whether a physician diagnoses the patient with pneumonia.
We train a model $f$
to predict pneumonia given an X-ray image.
Assume that in the training set 
$.1\%$ of patients have pneumonia.
We deploy $f$ in the clinic 
and for several months,
it reliably predicts roughly $.1\%$ positive. 

Fast-forward to January (distribution $Q$):
Running $f$ on the last week's data, 
we find that $5\%$ of patients 
are predicted to have pneumonia!
Because $f$ remains fixed,
the shift must owe to a 
change in the marginal $p(\vct x)$, violating the familiar iid assumption.
Absent familiar guarantees, 
we wonder: 
\emph{Is $f$ still accurate? 
What's the real current rate of pneumonia?}
Shouldn't our classifier, 
trained under an obsolete prior, 
underestimate pneumonia when uncertain?
Thus, we might suspect that the real prevalence is greater than $5\%$.

Given only \textbf{labeled training data}, 
and \textbf{unlabeled test data}, 
we desire to:
(i) detect distribution shift, 
(ii) quantify it, 
and (iii) correct our model
to perform well on the new data.
Absent assumptions
on how $p(y,\vct x)$ changes,
the task is impossible.
However, under assumptions
about what P and Q have in common,
we can still make headway.
Two candidates are \emph{covariate shift}
(where $p(y|\vct x)$ does not change) 
and \emph{label shift}
(where $p(\vct x | y)$ does not change). 
\citet{scholkopf2012causal}
observe that covariate shift corresponds 
to causal learning (predicting effects),
and label shift to anticausal learning (predicting causes).

%
%
We focus on label shift,
motivated by diagnosis 
(diseases cause symptoms)
and recognition tasks 
(objects cause sensory observations).
During a pneumonia outbreak, 
$p(y|\vct x)$ (e.g. flu given cough) 
might rise but the manifestations 
of the disease 
$p(\vct x |y)$ might not change.
Formally, under label shift, 
we can factorize the target distribution as
$$q(y, \vct x) = q(y) p(\vct x | y).$$

By contrast, under the \emph{covariate shift} assumption, $ q(y, \vct x) = q(\vct x) p(y | \vct x)$,
e.g. the distribution 
of radiologic findings $p(\vct x)$ changes, 
but the conditional probability of pneumonia $p(y| \vct x)$ remains constant.  
To see how this can go wrong, consider:
what if our only feature were \emph{cough}?
Normally, cough may not (strongly) indicate pneumonia.
But during an epidemic, 
$\P(\text{pneumonia}|\text{cough})$ might go up substantially.
Despite its importance, 
label shift is comparatively under-investigated,
perhaps because given samples 
from both $p(\vct x)$ and $q(\vct x)$, 
quantifying $q(\vct x)/p(\vct x)$ is more intuitive.

We introduce \algolong{} (\algo{})
to estimate label shift 
using a black box predictor $f$.
\algo{} estimates the
ratios $w_l = q(y_l)/p(y_l)$ 
for each label $l$,
requiring only that the
expected confusion
matrix is invertible \footnote{
For degenerate confusion matrices,
a variant using soft predictions may be preferable.}.
We estimate $\hat{\vct w}$ by solving a linear system $A \vct x = b$ where $A$ is the confusion matrix of $f$ estimated on training data (from P) and $b$ is the average output of $f$ calculated on test samples (from Q). We make the following contributions:
\begin{enumerate}
\item Consistency and error bounds for \algo{}.
\item Applications of \algo{} to statistical tests for detecting distribution label shift 
\item Model correction through importance-weighted Empirical Risk Minimization.
\item A comprehensive empirical validation of \algo{}.
\end{enumerate}
%
%

Compared to approaches based on Kernel Mean Matching (KMM) \citep{zhang2013domain},  EM \citep{chan2005word}, and Bayesian inference \citep{storkey2009training},  
\algo{} offers the following advantages:
(i) Accuracy does not depend on data dimensionality;
(ii) Works with arbitrary black box predictors, even biased, uncalibrated, or inaccurate models;
(iii) Exploits advances in deep learning while retaining theoretical guarantees:
better predictors provably lower sample complexity; and (iv) Due to generality, 
could be a standard diagnostic / corrective tool 
for arbitrary ML models.

\vspace{-3px}

\section{Prior Work}
\label{sec:background}
Despite its wide applicability, 
learning under label shift 
with unknown $q(y)$ 
remains curiously under-explored.
Noting the difficulty of the problem, \citet{storkey2009training}
proposes placing a (meta-)prior over $p(y)$
and inferring the posterior distribution from unlabeled test data.
Their approach requires explicitly estimating $p(\vct x|y)$, 
which may not be feasible in high-dimensional datasets. 
\citet{chan2005word} infer $q(y)$ 
using EM but their method also requires estimating $p(\vct x|y)$.
\citet{scholkopf2012causal} articulates connections 
between label shift and anti-causal learning 
and \citet{zhang2013domain} extend 
the kernel mean matching approach 
due to \cite{gretton2009covariate} 
to the label shift problem.
When $q(y)$ is known,
label shift simplifies to 
the problem of
changing base rates 
\cite{bishop1995neural,elkan2001foundations}.
Previous methods require estimating $q(\vct x)$, $q(\vct x)/p(\vct x)$, or $p(\vct x |y)$, often relying on kernel methods, 
which scale poorly with dataset size
and underperform on high-dimensional data. 


Covariate shift, also called \emph{sample selection bias},
is well-studied \citep{zadrozny2004learning,huang2007correcting,
sugiyama2008direct, gretton2009covariate}.
\citet{shimodaira2000improving}
proposed correcting models 
via weighting examples in ERM
by $q(\vct x)/p(\vct x)$.
Later works estimate
importance weights from the available data, 
e.g., \citet{gretton2009covariate}
propose kernel mean matching 
to re-weight training points.

The earliest relevant work to ours comes from econometrics and addresses the use of non-random samples to estimate behavior.
\citet{heckman1977sample} addresses sample selection bias, while \cite{manski1977estimation}
investigates estimating parameters under
\emph{choice-based} and \emph{endogenous stratified sampling},
cases analogous to a shift in the label distribution.
Also related, \citet{rosenbaum1983central} introduce propensity scoring to design
unbiased experiments. 
Finally, we note a connection 
to cognitive science work showing 
that humans classify items differently
depending on other items 
they appear alongside \citep{zhu2010cognitive}.

Post-submission, we learned of antecedents 
for our estimator 
in epidemiology \citep{buck1966comparison} 
and revisited
by \citet{forman2008quantifying,saerens2002adjusting}.
These papers do not develop our theoretical guarantees or explore the modern ML setting where $\vct x$ is massively higher-dimensional than $y$, 
bolstering the value of dimensionality reduction.

\vspace{-3px}

\section{Problem setup}
\label{sec:setup}
We use $\vct x\in \cX = \R^d$ and $y\in \cY$ to denote the feature and label variables. 
For simplicity, we assume that $\cY$ 
is a discrete domain equivalent to $\{1,2,...,k\}$. 
Let $P,Q$ be the source and target distributions defined on $\cX\times \cY$.
We use $p,q$ to denote the probability density function (pdf) or
probability mass function (pmf) associated with $P$ and $Q$
respectively. The random variable of interest is clear from
context. For example, $p(y)$ is the p.m.f. of $y\sim P$ and $q(\vct
x)$ is the p.d.f. of $\vct x\sim Q$. Moreover, $p(y=i)$ and $q(y=i)$
are short for $\P_P(y = i)$ and $\P_Q(y=i)$ respectively, where $\P(S)
:= \E[\mathbf{1}(S)]$ denotes the probability of an event $S$ and
$\E[\cdot]$ denotes the expectation. Subscripts $P$ and $Q$ on these
operators make the referenced distribution clear. 


In standard supervised 
learning, 
the learner observes training data 
$(\vct x_1,y_1),(\vct x_2,y_2),...,(\vct x_n,y_n)$ 
drawn iid from a training (or \emph{source}) distribution $P$. 
We denote the collection of feature vectors 
by $X\in\R^{n\times d}$ and the label by $\vct y$. 
Under Domain Adaptation (DA), 
the learner additionally observes 
a collection of samples $X' = [\vct x'_1; ... ;\vct x'_m]$
drawn iid from a test (or \emph{target}) distribution 
$Q$.
Our objective in DA is to
predict well for samples 
drawn from $Q$. 

In general, this task is 
impossible -- $P$ and $Q$ might not share support. 
This paper considers 
$3$ extra assumptions:
\begin{enumerate*}
	\item[A.1]   The \emph{label shift} (also known as \emph{target shift}) assumption
	$$
	p(\vct x|y)  = q(\vct x|y) \quad \forall \;x \in \cX,\; y\in \cY.
	$$ 
	\item[A.2] 
    For every $y\in\cY$ with $q(y)>0$ we require
    $p(y)>0$.\footnote{Assumes the absolute continuity of the (hidden)
      target label's distribution with respect to the source's, i.e.\ 
      $dq(y)/dp(y)$ exists.}
	\item[A.3] Access to a black box predictor $f: \cX\rightarrow
          \cY$ where the expected confusion matrix $\mat C_p(f)$
          is invertible.
	$$\mat C_P(f) := p(f(\vct x),y) \in\R^{|\cY|\times |\cY|}$$
\end{enumerate*}
We now comment on the assumptions. 
A.1 corresponds to anti-causal learning.
This assumption is strong but reasonable in many practical situations, 
including medical diagnosis, where diseases cause symptoms. 
It also applies when classifiers are trained on non-representative class distributions:
Note that while visions systems are commonly trained with balanced classes \cite{deng2009imagenet}, 
the true class distribution for real tasks is rarely uniform.
%
%
%

%
%

Assumption A.2 addresses identifiability,
requiring that the target label distribution's support be a subset of training distribution's.
For discrete $\cY$, this simply means 
that the training data should contain examples from every class.

%
%
Assumption A.3 requires 
that the expected predictor outputs for each class be linearly independent.
This assumption holds in the typical case
where the classifier predicts class $y_i$ 
more often given images actually belong to $y_i$
than given images from any other class $y_j$.
In practice, $f$ could be a neural network, 
a boosted decision-tree or any other classifier 
trained on a holdout training data set. 
We can verify at training time 
that the empirical estimated normalized confusion matrix is invertible. 
Assumption A.3 generalizes naturally to  soft-classifiers, 
where $f$ outputs a probability distribution 
supported on $\cY$. 
Thus \algo{} can be applied 
even when the confusion matrix is degenerate.


We wish to estimate 
$w(y):= q(y)/p(y)$ for every $y\in\cY$ 
with training data, unlabeled test data 
and a predictor $f$. 
This estimate enables DA techniques 
under the importance-weighted ERM framework, 
which solves
 $\min  \sum_{i=1}^n w_i \ell(y_i,\vct x_i)$, using
 $w_i  = q(\vct x_i,y_i)/p(\vct x_i,y_i)$.
Under the label shift assumption, 
the importance weight $w_i = q(y_i)/ p(y_i)$. 
This task isn't straightforward 
because we don't observe samples from $q(y)$. 
\vspace{-3px}

\section{Main results}
\label{sec:theory}

We now derive the main results for estimating $w(y)$ and
$q(y)$. Assumption A.1 has the following implication:
\begin{lemma}
\label{lem:matching_confusion}
Denote by $\hat{y}  = f(\vct x)$ the output of a 
fixed function $f: \cX \rightarrow \cY$. 
If Assumption A.1 holds, then 
$$
q(\hat{y} | y)  =  p(\hat{y}| y).
$$
\end{lemma}
The proof is simple: 
recall that $\hat{y}$ depends on $y$ 
only via $\vct x$. 
By A.1, 
$p(\vct x|y) = q(\vct x|y)$ and thus $q(\hat{y} | y)  =  p(\hat{y}| y)$.

Next, combine the law of total probability and Lemma~\ref{lem:matching_confusion} and we arrive at
\begin{align}
 q(\hat{y}) = & \sum_{y\in\cY} q(\hat{y} | y) q(y) \nonumber \\ 
  = & \sum_{y\in\cY} p(\hat{y} | y) q(y) = \sum_{y\in\cY} p(\hat{y},y)
      \frac{q(y)}{p(y)}.
 \label{eq:main_idea}
\end{align}
We estimate $p(\hat{y}|y)$ and $p(\hat{y},y)$ using $f$ and 
data from source distribution $P$,
and $q(\hat{y})$ with \emph{unlabeled} test data 
drawn from target distribution $Q$. 
This leads to a novel method-of-moments approach 
for consistent estimation
of the shifted label distribution $q(y)$ and the weights $w(y)$. 

Without loss of generality, we assume $\cY =\{1,2,...,k\}$. 
Denote by $\vct \nu_y, \vct \nu_{\hat{y}}, \hat{\vct \nu}_{\hat{y}}, \vct \mu_y, \vct \mu_{\hat{y}}, \hat{\vct \mu}_{\hat{y}},  
\vct w \in \R^k$ 
moments of $p$, $q$, and their plug-in estimates, defined via
\begin{align*}
  [\vct \nu_y]_i & = p(y=i)& 
  [\vct \mu_y]_i & = q(y=i) \\
  [\vct \nu_{\hat y}]_i & = p(f(\vct x)=i) & [\vct \mu_{\hat y}]_i & = q(f(\vct x)=i) \\ 
 [\hat{\vct \nu}_{\hat y}]_i & = \frac{\sum_j 
 \mathbbm{1}  \{f(\vct x_j)=i\}}{n} &
[\hat{\vct \mu}_{\hat y}]_i & = \frac{\sum_j 
  \mathbbm{1} \{f(\vct x_j')=i\}}{m} 
\end{align*}
and $[\vct w]_i = q(y=i)/p(y=i)$. 
Lastly define the covariance matrices 
$\mat C_{\hat{y},y}, \mat C_{\hat{y}|y}$ and $\hat{\mat C}_{\hat{y},y}$ in $\R^{k \times k}$ 
via 
\begin{align*}
  [\mat C_{\hat{y},y}]_{ij} & =  p(f(\vct x)=i, y=j)\\
  [\mat C_{\hat{y}|y}]_{ij} & =  p(f(\vct x)=i | y=j)\\
  [\hat{\mat C}_{\hat{y},y}]_{ij} & = \frac{1}{n} \sum_{l} \mathbbm{1} \{f(\vct x_l) = i \text{ and } y_l = j\} 
\end{align*}
We can now rewrite Equation \eqref{eq:main_idea} 
in matrix form:
\begin{align*}
\vct \mu_{\hat{y}} =  {\mat C}_{\hat{y} | y}\vct \mu_y  = \mat C_{\hat{y},y} \vct w
\end{align*}
Using plug-in maximum likelihood estimates of the above quantities yields the estimators
$$\hat{\vct w}  =  \hat{\mat C}_{\hat{y},y}^{-1} 
\hat{\vct \mu}_{\hat{y}}
\text{ and }
\hat{\vct \mu}_{y} =  
\diag(\hat{\vct \nu}_y) \hat{\vct w},
$$
where $\hat{\vct \nu}_y$ 
is the plug-in estimator of $\vct \nu_y$. 

Next, we establish that the estimators are consistent.
\begin{proposition}[Consistency]\label{prop:consistency}
	If Assumption A.1, A.2, A.3 are true, then as $ n,m\rightarrow \infty$,
	$\hat{\vct w} \overset{\text{a.s.}}{\longrightarrow} \vct w$ and  
	$\hat{\vct \mu}_{y}  \overset{\text{a.s.}}{\longrightarrow}  \vct \mu_y.$
\end{proposition}
The proof (see Appendix~\ref{app:proofs}) 
uses the First Borel-Cantelli Lemma 
to show that the probability 
that the entire sequence of empirical confusion matrices with data size $n+1,...,\infty$ 
are \emph{simultaneously} invertible converges to $1$,  thereby enabling us to use 
the continuous mapping theorem 
after applying the strong law of large numbers 
to each component.

We now address our estimators' convergence rates.
\begin{theorem}[Error bounds]\label{thm:estimating_ratios}
Assume that A.3 holds robustly.
Let $\sigma_{\min}$ be the smallest eigenvalue 
of ${\mat C}_{\hat{y},y}$.
 There exists a constant $C > 0$ 
 such that for all $n>80\log(n) \sigma_{\min}^{-2}$, 
 with probability at least $1-3kn^{-10} - 2km^{-10} $ we have 
 \begin{align}
   \label{eq:ratiobound}
   \|\hat{\vct w} - \vct w\|_2^2 & \leq
   \frac{C}{\sigma_{\min}^2}\left(\frac{\|\vct w\|^2\log
                                   n}{n}+\frac{k\log m}{m}\right) 
   \\
   \label{eq:mu_y_errorbound}
   \|\hat{\vct \mu}_y - \vct \mu_y\|^2  & \leq  
   \frac{C\|\vct w\|^2\log n}{n}  + \|\vct \nu_y\|_{\infty}^2 
                                          \|\hat{\vct w} - \vct w\|_2^2
 \end{align}
\end{theorem}
The bounds give practical insights 
(explored more in Section \ref{sec:discussion}). 
In \eqref{eq:ratiobound},
the square error depends on the sample size and is proportional to $1/n$ (or $1/m$).
There is also a $\|\vct w\|^2$ term 
that reflects 
how different the source and target distributions are.
In addition, $\sigma_{\min}$ 
reflects the quality of the given classifier $f$. 
For example, if $f$ is a perfect classifier, 
then $\sigma_{\min}  = \min_{y\in\cY} p(y)$. 
If $f$ cannot distinguish between certain classes at all, 
then ${\mat C}_{\hat{y},y}$
will be low-rank, $\sigma_{\min} = 0$, 
and the technique is invalid, as expected. 

We now parse the error bound
of $\hat{\vct \mu}_y$ in \eqref{eq:mu_y_errorbound}. 
The first term $\|\vct w\|^2/n$ 
is required even if we observe 
the importance weight $\vct w$ exactly. 
The second term captures 
the additional error due to the fact 
that we estimate $\vct w$ 
with predictor $f$. 
Note that $\|\vct \nu_y\|_{\infty}^2 \leq 1$ 
and can be as small as $1/k^2$ when $p(y)$ 
is uniform.
Note that when $f$
correctly classifies each class 
with the same probability, 
e.g. $0.5$, 
then $\|\vct \nu_y\|^2/\sigma_{\min}^2$ 
is a constant and the bound cannot be improved. 
\vspace{-11px}
\begin{proof}[Proof of Theorem~\ref{thm:estimating_ratios}]
	Assumption A.2 ensures that $\vct w < \infty$. 
	\begin{align*}
	\hat{\vct w}  &=  \hat {\mat C}_{\hat{y},y}^{-1} \hat{\vct \mu}_{\hat{y}} =  ({\mat C}_{\hat{y},y}+E_1)^{-1}(\vct \mu_{\hat{y}} + E_2)\\
    &=  \vct w + [({\mat C}_{\hat{y},y}+E_1)^{-1} - {\mat C}_{\hat{y},y}^{-1}] \vct \mu_{\hat{y}} + ({\mat C}_{\hat{y},y}+E_1)^{-1}E_2
	\end{align*}
By completing the square and Cauchy-Schwartz inequality,
	\begin{align*}
	\| \hat{\vct w} - \vct w \|^2  &\leq 2 \vct \mu_{\hat{y}}^T[({\mat C}_{\hat{y},y}+E_1)^{-1}\\
&- {\mat C}_{\hat{y},y}^{-1}]^T[({\mat C}_{\hat{y},y}+E_1)^{-1} - {\mat C}_{\hat{y},y}^{-1}] \vct \mu_{\hat{y}} \\
&+ 2 E_2 [({\mat C}_{\hat{y},y}+E_1)^{-1}]^T({\mat C}_{\hat{y},y}+E_1)^{-1}E_2.
	\end{align*}
	By Woodbury matrix identity, we get that 
	$$\hat {\mat C}_{\hat{y},y}^{-1}   =  {\mat C}_{\hat{y},y}^{-1}    +  {\mat C}_{\hat{y},y}^{-1}  [ E_1^{-1} + {\mat C}_{\hat{y},y}^{-1} ]^{-1} {\mat C}_{\hat{y}, y}^{-1}.$$
Substitute into the above inequality and use \eqref{eq:main_idea} we get
\begin{align}
      \label{eq:square-error-expansion}
  \|\hat{\vct w} - \vct w \|^2  \leq & 2 \vct w \left\{[ E_1^{-1} +
                                  {\mat C}_{\hat{y},y}^{-1}
                                  ]^{-1}\right\}^T [{\mat
                                  C}_{\hat{y},y}^{-1} ]^T \times \\
  \nonumber
  & {\mat C}_{\hat{y},y}^{-1}[ E_1^{-1} + {\mat C}_{\hat{y},y}^{-1} ]^{-1} \vct w \\
  \nonumber
    &+ 2 E_2^T [({\mat C}_{\hat{y},y}+E_1)^{-1}]^T({\mat C}_{\hat{y},y}+E_1)^{-1}E_2
\end{align}
We now provide a high probability bound 
on the Euclidean norm of $E_2$, 
the operator norm of $E_1$, 
which will give us an operator norm bound of 
$[ E_1^{-1} + {\mat C}_{\hat{y},y}^{-1} ]^{-1}$ 
and $({\mat C}_{\hat{y},y}+E_1)^{-1}$ 
under our assumption on $n$, 
and these will yield a high probability bound 
on the square estimation error.
	
\textbf{Operator norm of $E_1$. }
Note that $\hat {\mat C}_{\hat{y},y} = \frac{1}{n} \sum_{i=1}^n  \vct e_{f(x_i)} \vct e_{y_i}^T $, 
where $\vct e_y$ is the standard basis with $1$ at the index of $y\in \cY$ and $0$ elsewhere.
Clearly, $\E \vct e_{f(x_i)} \vct e_{y_i}^T  = {\mat C}_{\hat{y},y} $.  Denote $\mat Z_i := \vct e_{f(x_i)} \vct e_{y_i}^T  - {\mat C}_{\hat{y},y} $.  
	Check that $\|\mat Z_i\|_2 \leq \|\mat Z_i\|_F \leq \|\mat Z_i\|_{1,1} \leq 2$, $\max\{\|\E [\mat Z_i \mat Z_i^T]\|,\|\E [\mat Z_i^T \mat Z_i]\|\} \leq 1$, by matrix Bernstein inequality (Lemma~\ref{lem:matrixbernstein}) 
we have for all $t\geq 0$:
	$$
	\P(\|E_1\| \geq t/n)  \leq 2k e^{-\frac{t^2}{n + 2t/3} }.
	$$
Take $t = \sqrt{20n\log n}$ and use the assumption that $n\geq 4\log n / 9$ (which holds under our assumption on $n$ since $\sigma_{\min} < 1$). Then with probability at least $1-2kn^{-10}$
	$$\|E_1\|\leq \sqrt{\frac{20\log n}{n}}.$$
	
Using the assumption on $n$, we have $\|E_1\| \leq \sigma_{\min}/2$
	$$
	\| [ E_1^{-1} + {\mat C}_{\hat{y},y}^{-1} ]^{-1} \| \leq 2\|E_1\|\leq \frac{2\sqrt{20\log n}}{\sqrt{n}}.
	$$
Also, we have 
	$
	\|({\mat C}_{\hat{y},y}+E_1)^{-1}\|  \leq \frac{2}{\sigma_{\min}}
	$.
	
\textbf{Euclidean norm of $E_2$. }
Note that $[E_2]_l = \frac{1}{m}\sum_{i=1}^m \mathbf{1}(f(\vct x_i')=l) - q(f(\vct x_i')=l)$. By the standard Hoeffding's inequality and union bound argument, we have that with probability larger than $1-2km^{-10}$
\begin{equation*}
\|E_2\| = \|\vct \mu_{\hat{y}}  - \hat{\vct \mu}_{\hat{y} }\|_2  \leq \frac{\sqrt{10k \log m}}{\sqrt{m}}	
\end{equation*}
\vspace{-10px}
	
Substitute into Equation \ref{eq:square-error-expansion}, we get
	\begin{equation}\label{eq:errbound_deriv}
	\| 	\hat{\vct w} - 	\vct w \|^2 \leq \frac{80\log n}{\sigma_{\min}^2 n} \|\vct w\|^2  +  \frac{80k\log m}{\sigma_{\min}^2 m },
	\end{equation}
	which holds with probability $1-2kn^{-10} - 2km^{-10}$.
We now turn to $\hat{\vct \mu}_y$. Recall that $\hat{\vct \mu}_y  =  \diag(\hat{\vct \nu}_y) \hat{\vct w}$. Let the estimation error of $\hat{\vct \nu}_y$ be $E_0.$
	\begin{align*}
	\hat{\vct \mu}_y  =&  \vct \mu_y  +  \diag(E_0)\vct w  +  \diag(\vct \nu_y)(\hat{\vct w} -\vct w) \\  &+    \diag(E_0)(\hat{\vct w} -\vct w).
    \vspace{-10px}
	\end{align*}
By Hoeffding's inequality $\|E_0\|_\infty \leq \sqrt{\frac{20\log n}{n}}$ with probability larger than $1-kn^{-10}$. Combining with 
\eqref{eq:errbound_deriv} yields
	$$
	\|\hat{\vct \mu}_y  - \vct \mu_y \|^2  \leq \frac{20\|\vct w\|^2\log n}{n} +  \|\vct \nu_y\|_\infty^2	\| 	\hat{\vct w} - 	\vct w \|^2 + O(\frac{1}{n^2})
	$$
		which holds with probability $1-3kn^{-10} - 2km^{-10}$.
\end{proof}
\vspace{-10px}

\vspace{-3px}

\section{Application of the results}
\label{sec:application}
\begin{figure*}[tb]
  \begin{subfigure}[t]{0.32\textwidth}
    \includegraphics[width=.97\textwidth]{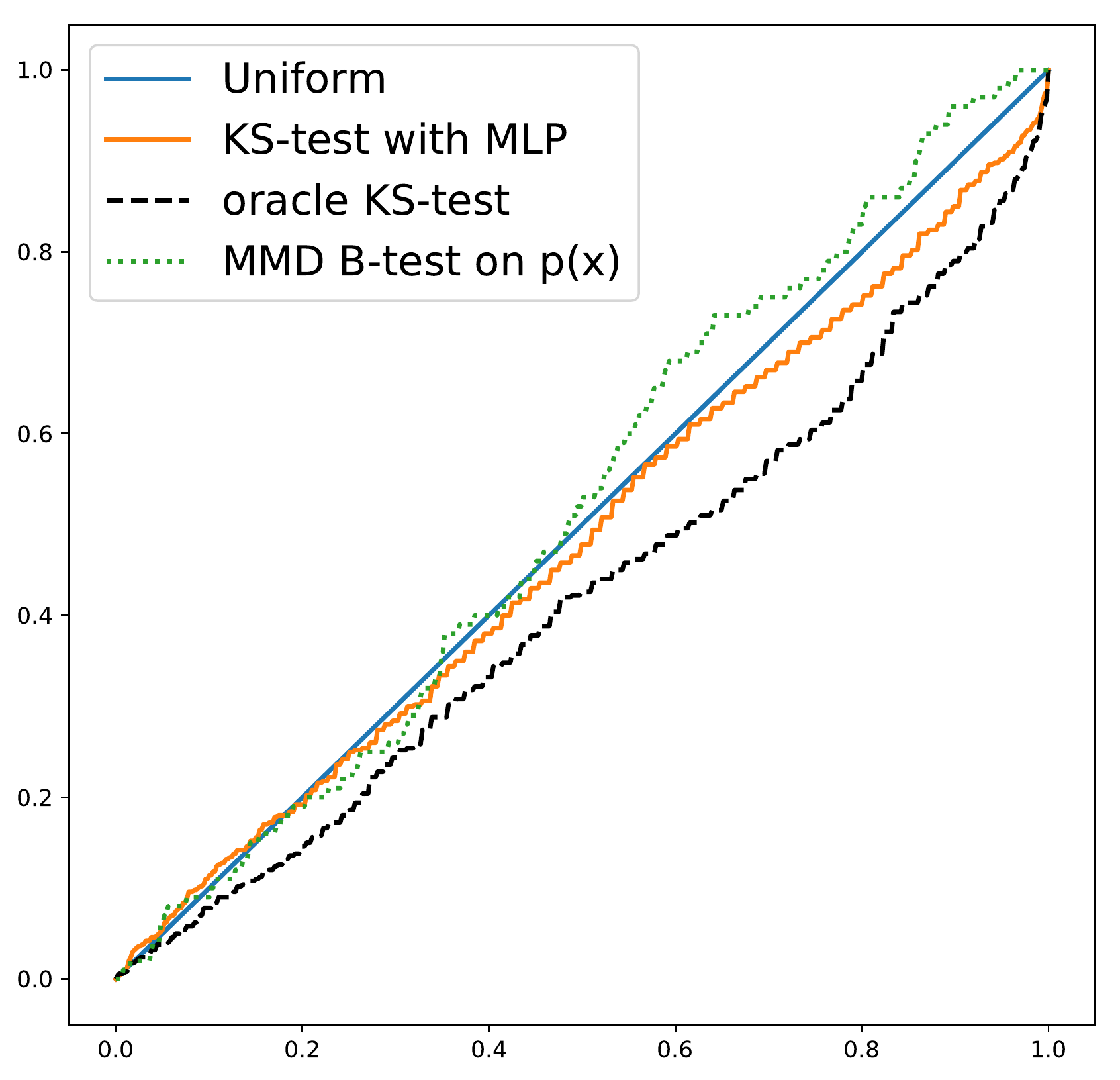}
    \caption{CDF of $p$-value at $\delta=0$ \label{fig:alpha_level_control}}
  \end{subfigure}
  \begin{subfigure}[t]{0.32\textwidth}
    \includegraphics[width=.97\textwidth]{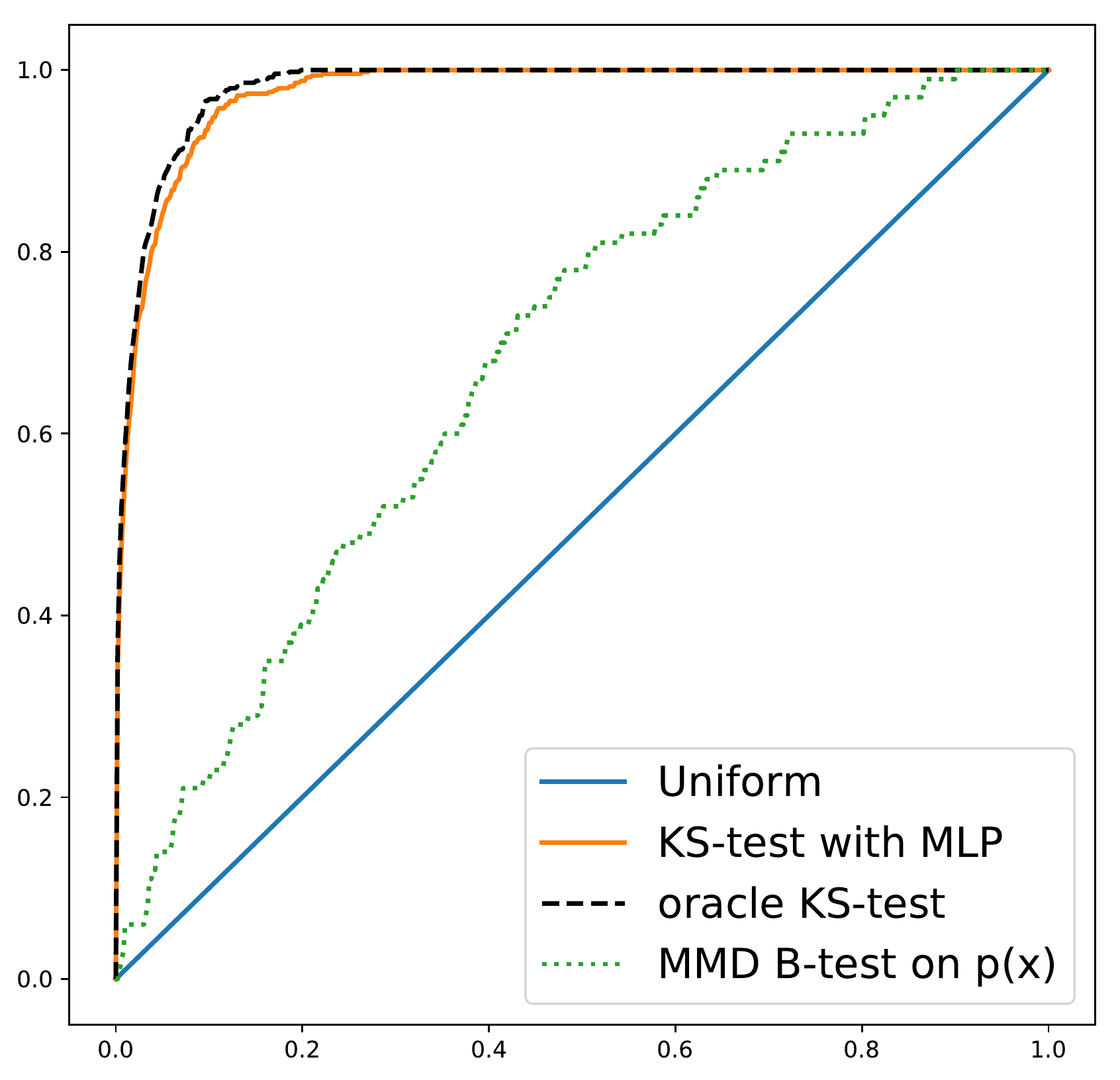}
    \caption{CDF of $p$-value at $\delta = 0.6$ \label{fig:power_at_0.6}}
  \end{subfigure}
  \begin{subfigure}[t]{0.32\textwidth}
    \includegraphics[width=.97\textwidth]{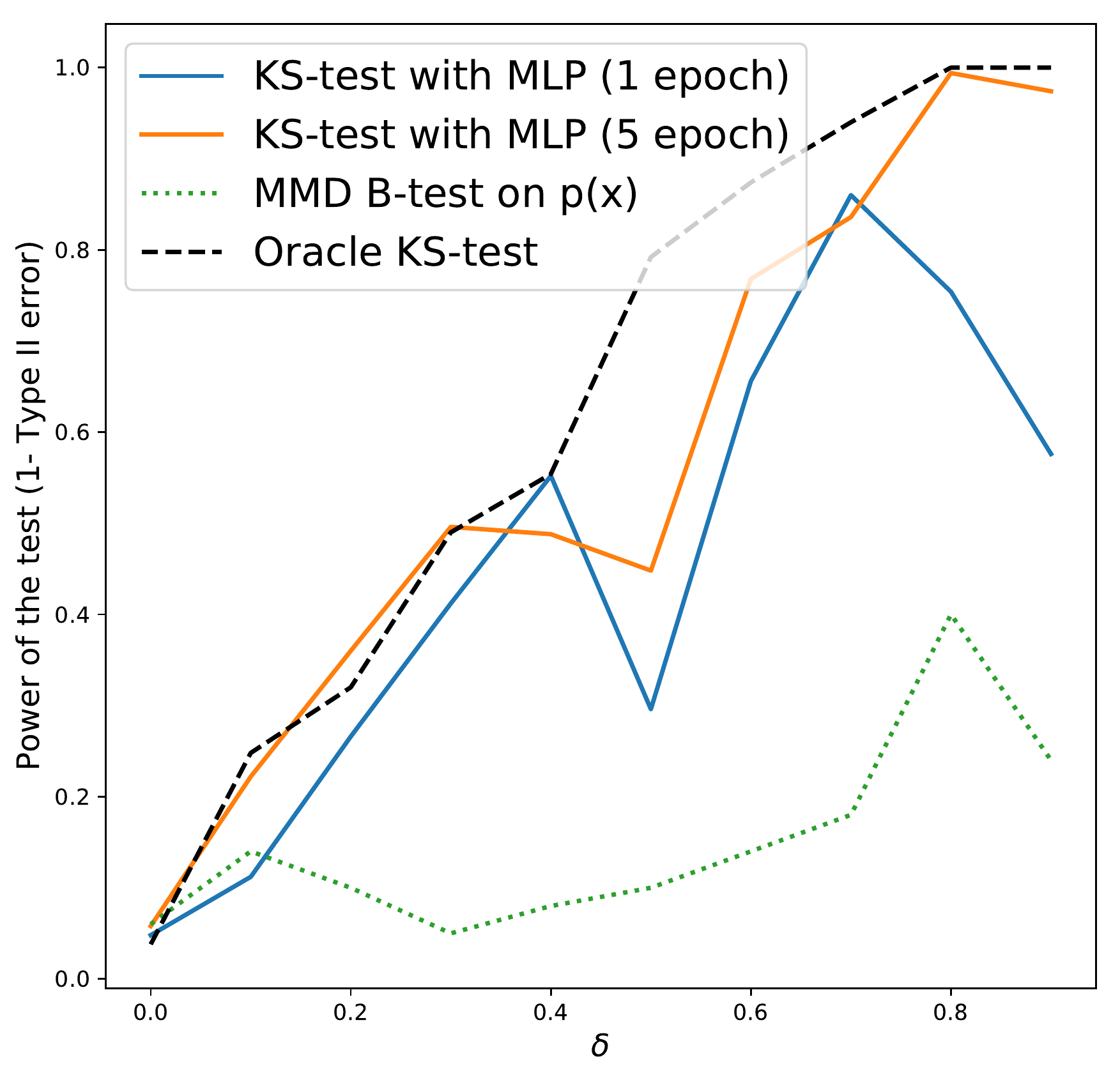}
    \caption{Power at $\alpha=0.05$ \label{fig:power_vs_oracle}}
  \end{subfigure}
  \caption{Label-shift detection on MNIST. 
		Pane \ref{fig:alpha_level_control} illustrates that Type I error is correctly controlled absent label shift. 
		Pane \ref{fig:power_at_0.6} illustrates high power 
		under mild label-shift. 
		Pane \ref{fig:power_vs_oracle} 
		shows increased power 
		for better classifiers. 
		We compare to
        kernel two-sample tests \citep{zaremba2013b}  
		and an (infeasible) oracle 
        two sample test 
        that directly tests 
        $p(y) = q(y)$ 
		with samples from each. 
		The proposed test beats directly testing 
        in high-dimensions 
		and nearly matches the oracle.
	\label{fig:exp_hypothesis_testing}}
    \vspace{-10px}
\end{figure*}

\subsection{Black Box Shift Detection (BBSD)}\label{sec:detection}
Formally, detection can be cast as a hypothesis testing problem 
where the null hypothesis is
$\mathbf{H}_0:  q(y) = p(y)$
and the alternative hypothesis is that
$
\mathbf{H}_1:  q(y) \neq p(y).
$
Recall that we observe neither $q(y)$ 
nor any samples from it. 
However, we do observe unlabeled data from the target distribution 
and our predictor $f$.

\begin{proposition}[Detecting label-shift]
Under Assumption A.1, A.2 and for each classifier $f$ satisfying A.3
  we have that $q(y) = p(y)$ if and only if $p(\hat{y}) =
  q(\hat{y})$.
\end{proposition}
\begin{proof}
  Plug  $P$ and $Q$ into \eqref{eq:main_idea} and apply
  Lemma~\ref{lem:matching_confusion} with assumption A.1.  
The result follows directly from our analysis in the proof of Proposition~\ref{prop:consistency} 
that shows $p(\hat{y},y)$ is invertible 
under the assumptions A.2 and A.3.
\end{proof}
Thus, under weak assumptions, 
we can test $\mathbf{H}_0$ 
by running two-sample tests 
on readily available samples from $p(\hat{y})$ and $q(\hat{y})$.
Examples include the Kolmogorov-Smirnoff test, 
Anderson-Darling or the Maximum Mean Discrepancy.
In all tests, asymptotic distributions are known 
and we can almost perfectly control the  Type I error. 
The power of the test ($1$-Type II error) 
depends on the classifier's performance on distribution $P$,
thereby allowing us to leverage recent progress in deep learning 
to attack the classic problem 
of detecting non-stationarity in the data distribution. 

One could also test whether $p(x) = q(x)$. 
Under the label-shift assumption 
this is implied by $q(y) = p(y)$. 
The advantage of testing the distribution of 
$f(\vct x)$ instead of $\vct x$ 
is that we only need to deal 
with a one-dimensional distribution.
Per theory and experiments in \cite{ramdas2015decreasing} 
two-sample tests in high dimensions
are exponentially harder. 

One surprising byproduct is that 
we can sometimes use this approach 
to detect covariate-shift, 
concept-shift, and more general 
forms of nonstationarity. 
\begin{proposition}[Detecting general  nonstationarity]
For any fixed measurable $f:\cX \rightarrow \cY$
	$$P=Q \implies p(x) = q(x) \implies p(\hat{y}) = q(\hat{y}).$$
\end{proposition}
  This follows directly 
from the measurability of $f$. 

While the converse is not true in general,  $p(\hat{y}) = q(\hat{y})$ does imply that for every measurable $\cS\subset \cY$,
$$
 q(x\in f^{-1}( \cS) ) = p(x\in f^{-1}(\cS)).
$$
This suggests that testing $\hat{\mathbf{H}}_0:  p(\hat{y}) =
q(\hat{y})$ may help us to determine if there's sufficient statistical evidence 
that domain adaptation techniques are required.


\begin{figure*}[tb]
  \includegraphics[width=0.32\textwidth]{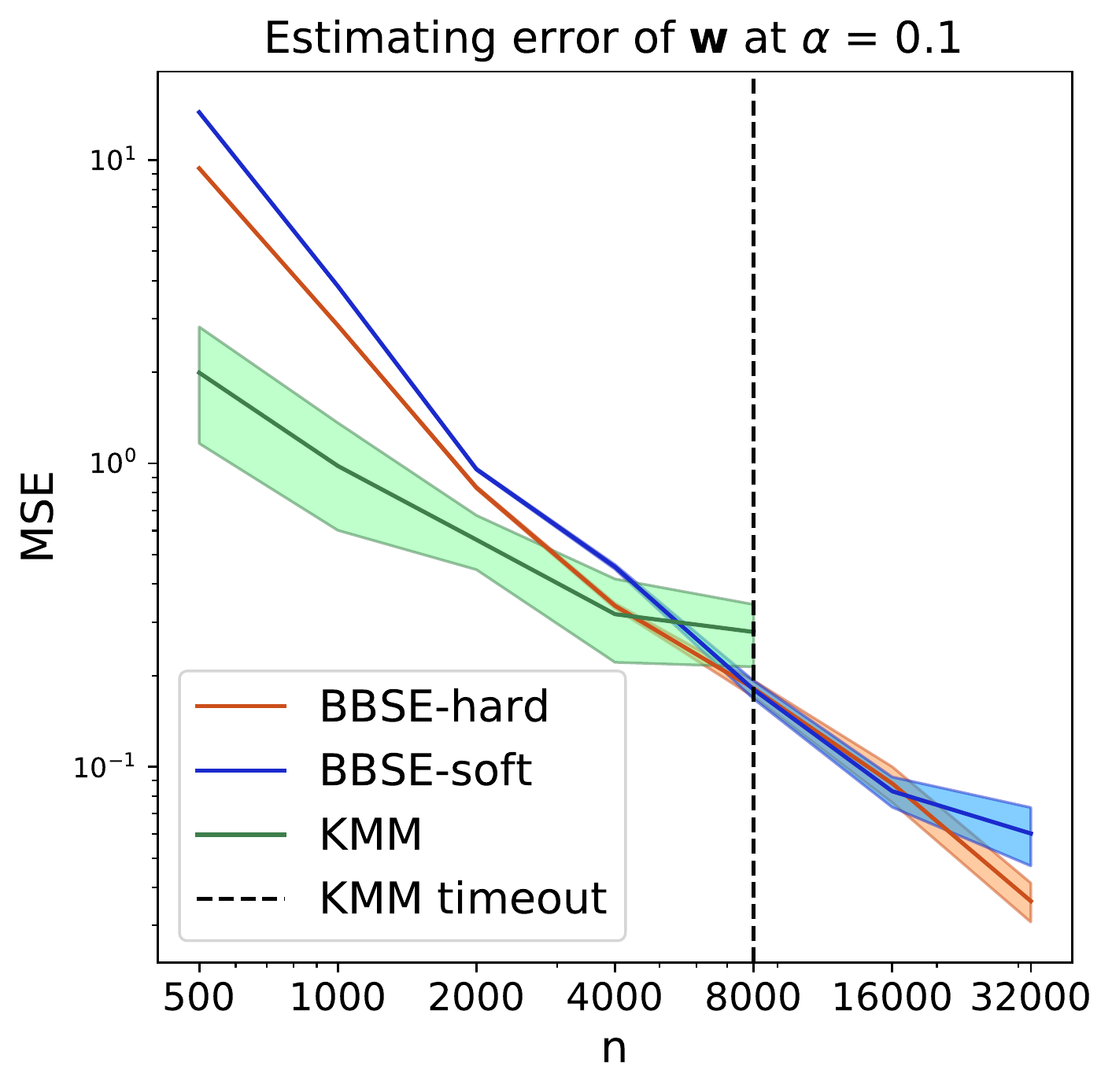}
  \hfill
  \includegraphics[width=0.32\textwidth]{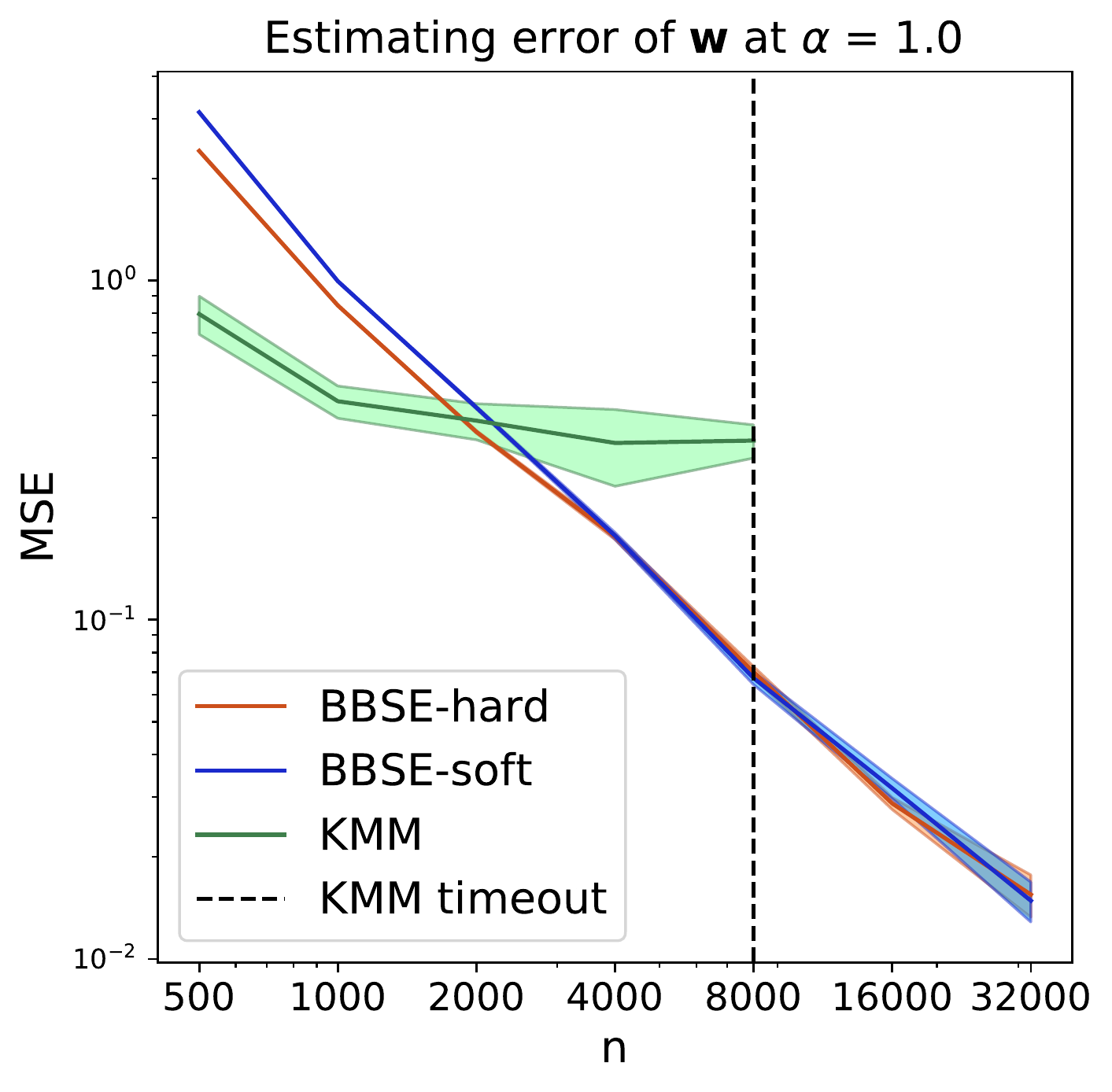}
  \hfill
  \includegraphics[width=0.32\textwidth]{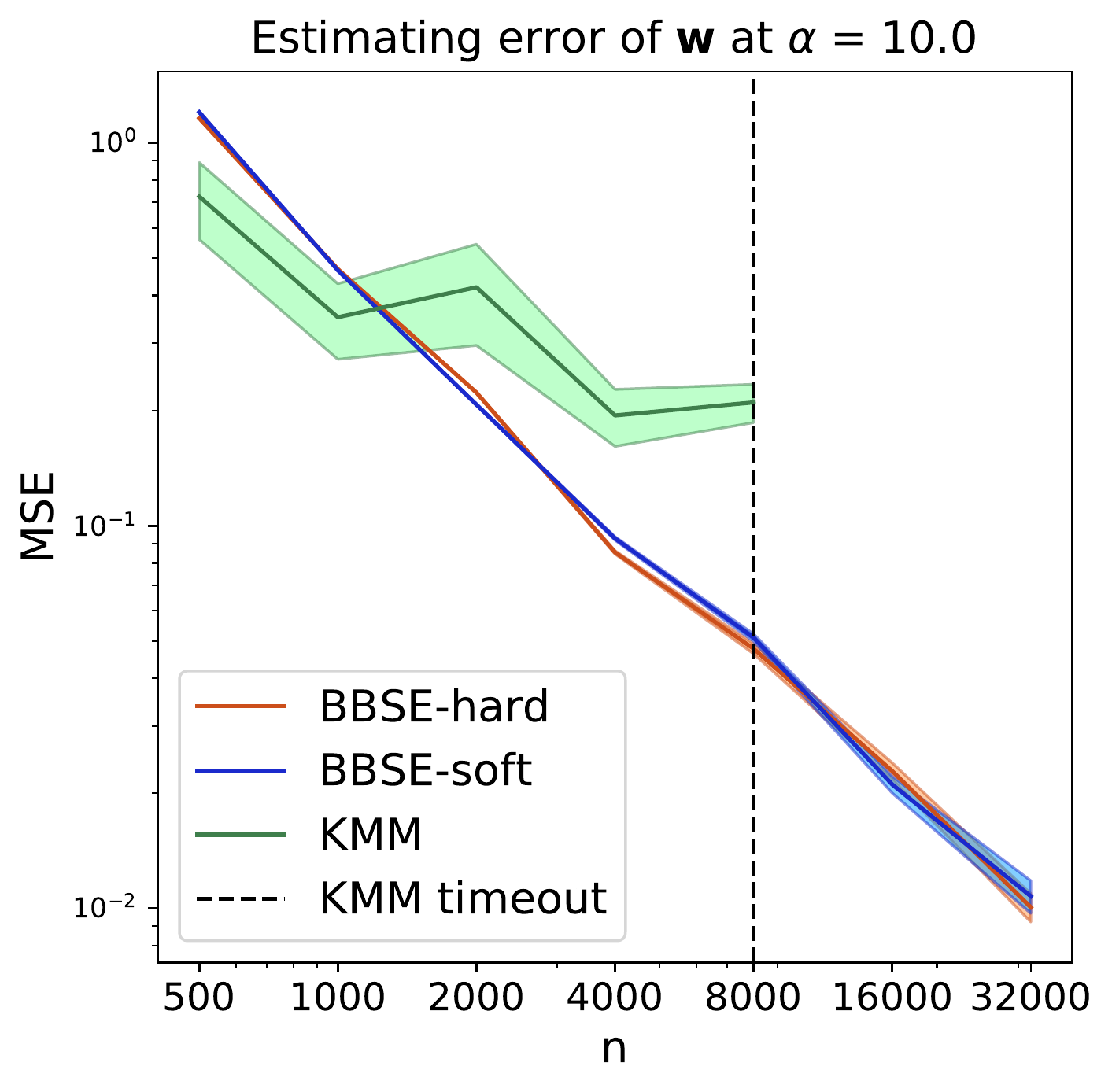}
  \\
  \includegraphics[width=0.32\textwidth]{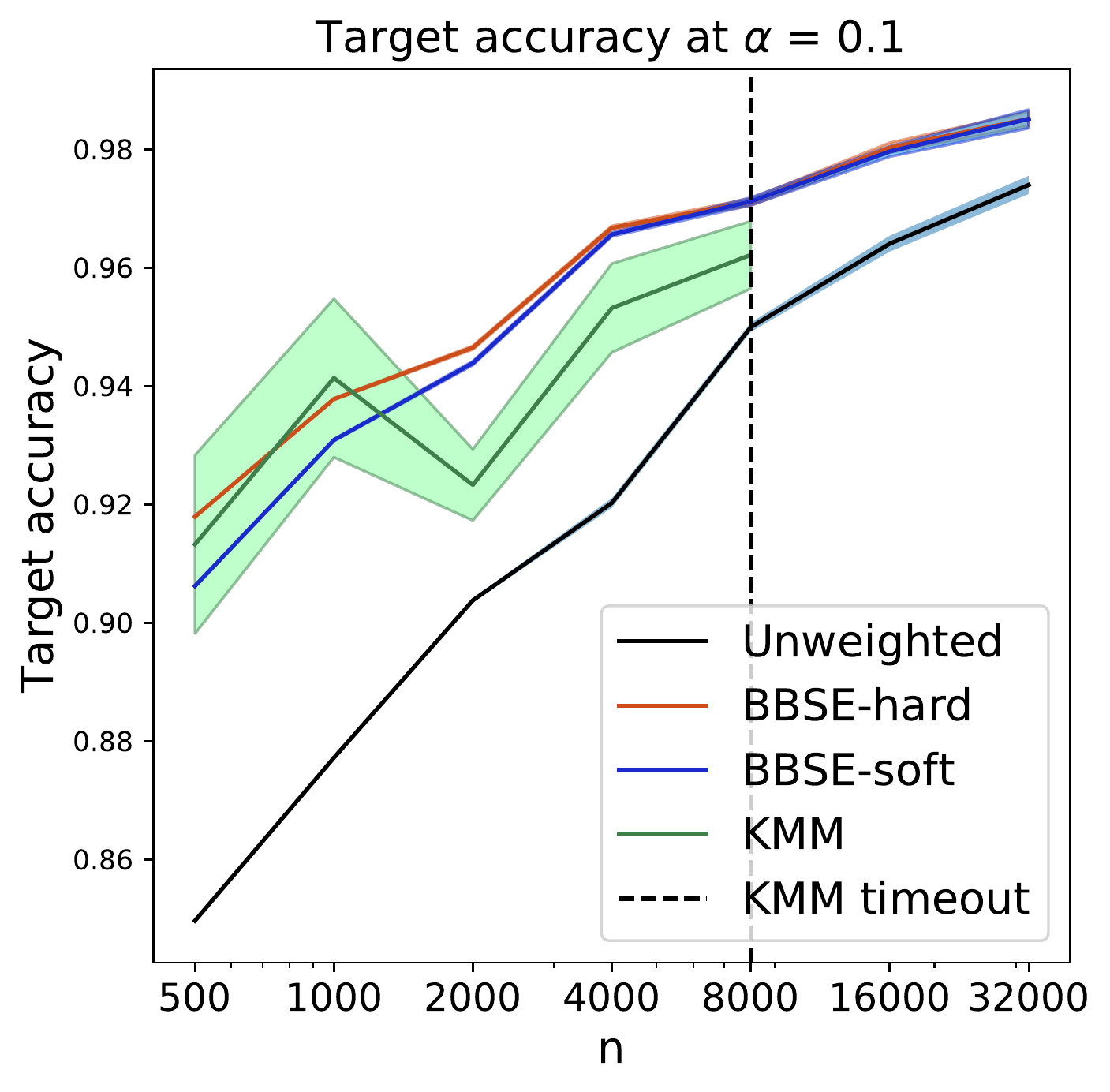}
  \hfill
  \includegraphics[width=0.32\textwidth]{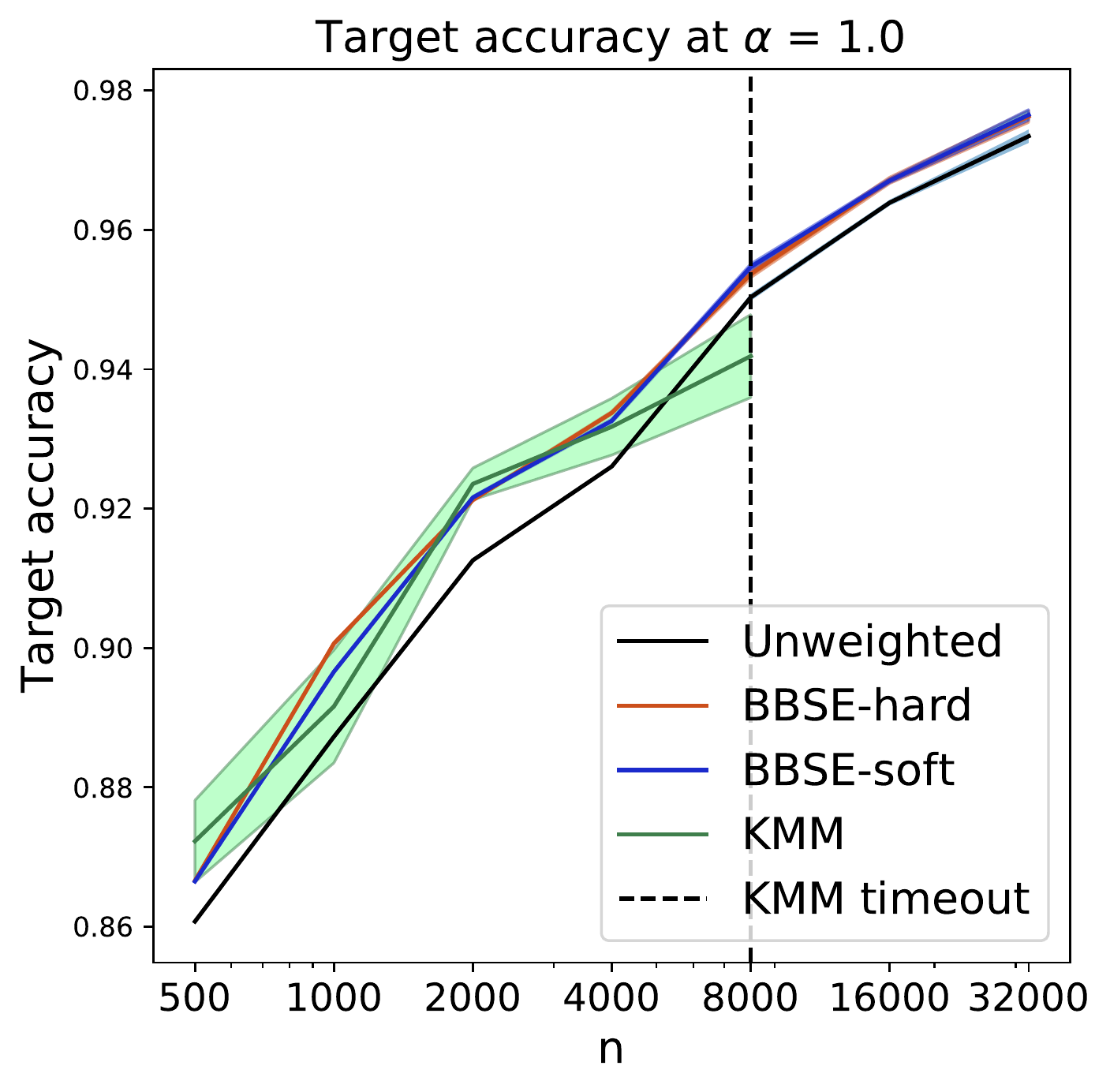}
  \hfill
  \includegraphics[width=0.32\textwidth]{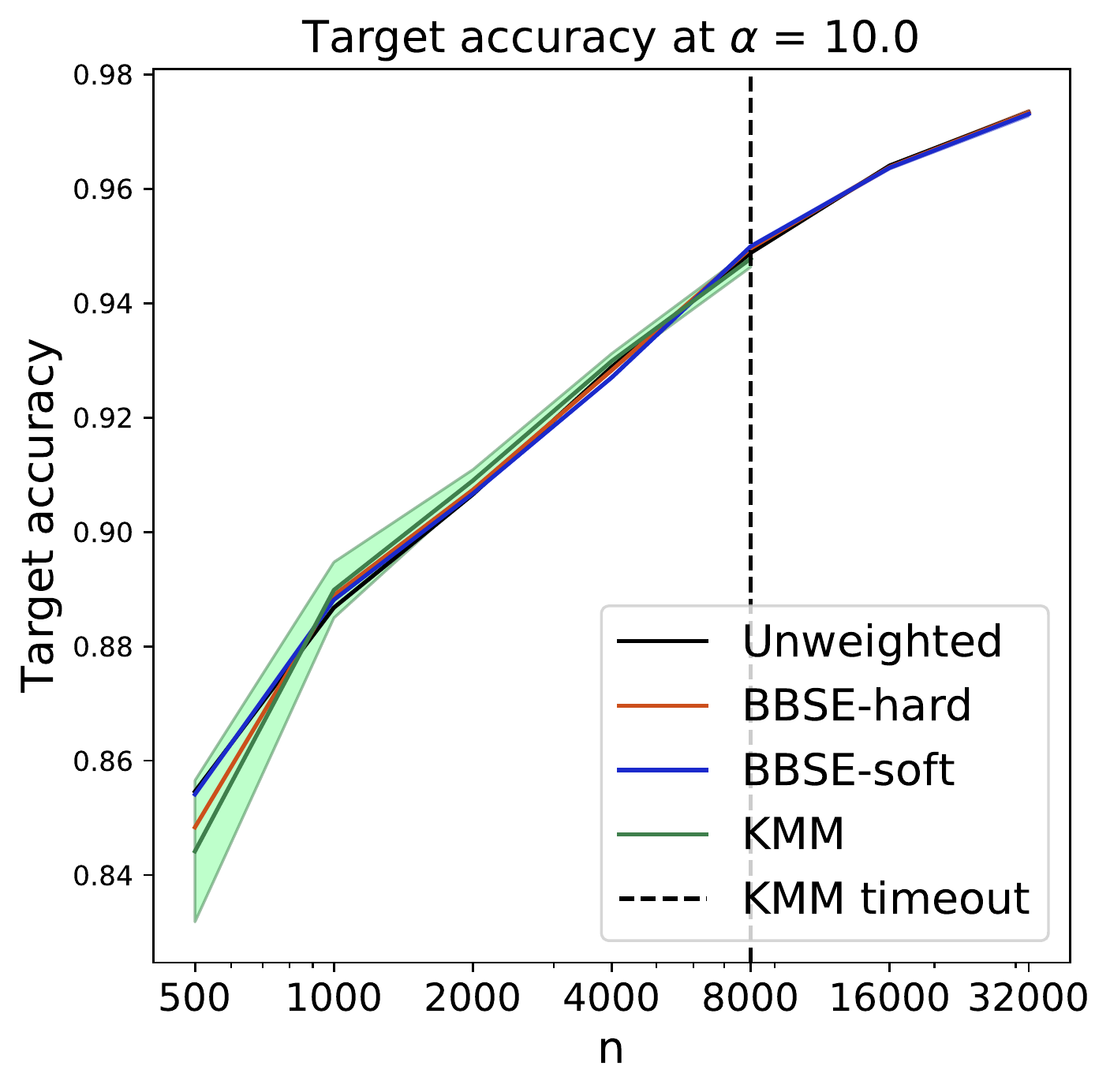}
  \caption{
    Estimation error (top row) 
    and correction accuracy (bottom row)
    vs dataset size on MNIST data 
    compared to KMM \citep{zhang2013domain}
    under Dirichlet shift (left to right) 
    with $\alpha = \{.1,1.0,10.0\}$ (smaller $\alpha$ means larger shift).
    \algo{} confidence interval on $20$ runs,
    KMM on $5$ runs due to computation; 
    $n=8000$ is largest feasible KMM experiment. 
    \label{fig:exp_increase_n_dirichlet}
  }
\end{figure*}

\begin{figure*}	
  \includegraphics[width=0.32\textwidth]{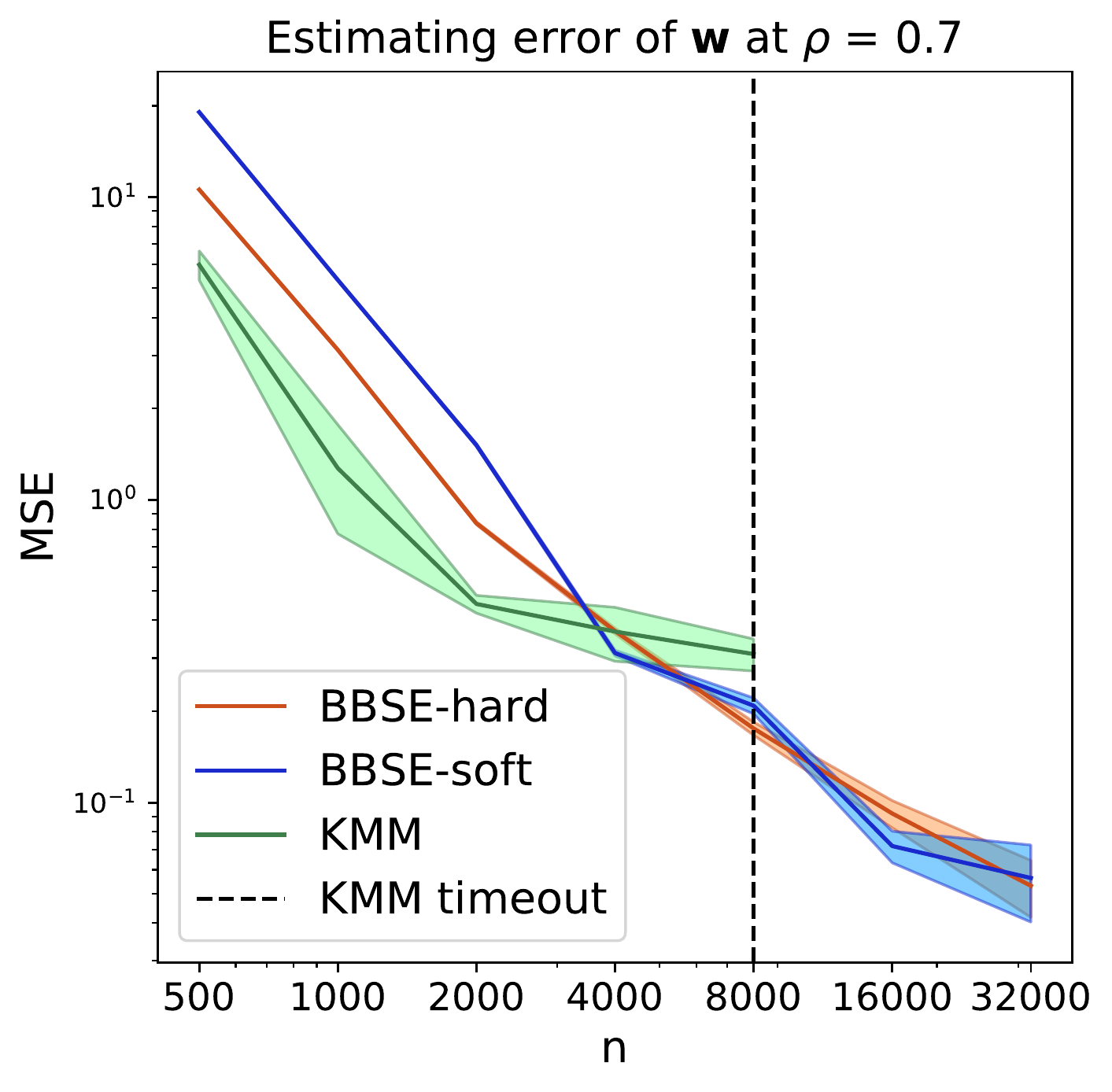}
  \hfill
  \includegraphics[width=0.32\textwidth]{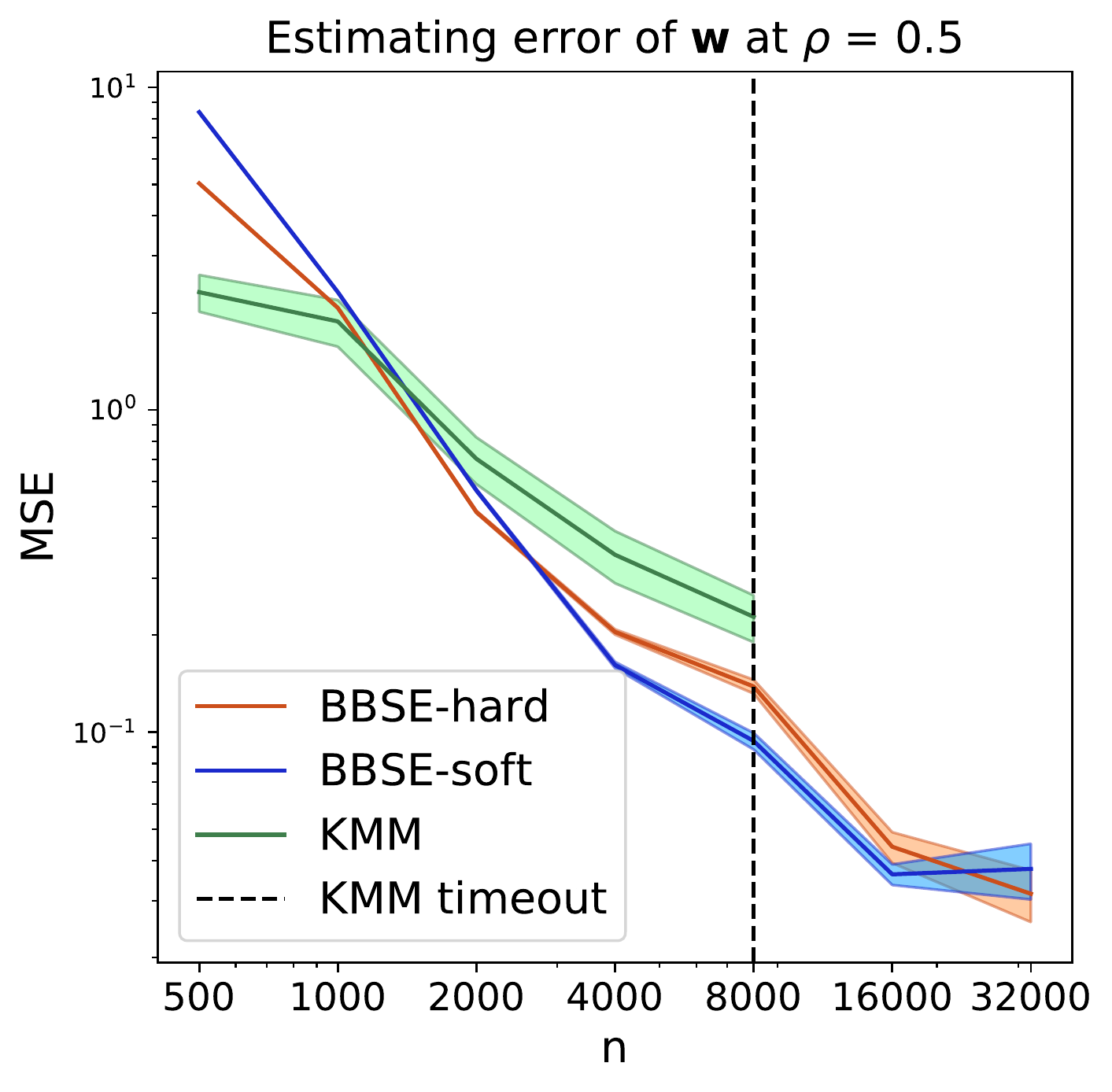}
  \hfill
  \includegraphics[width=0.32\textwidth]{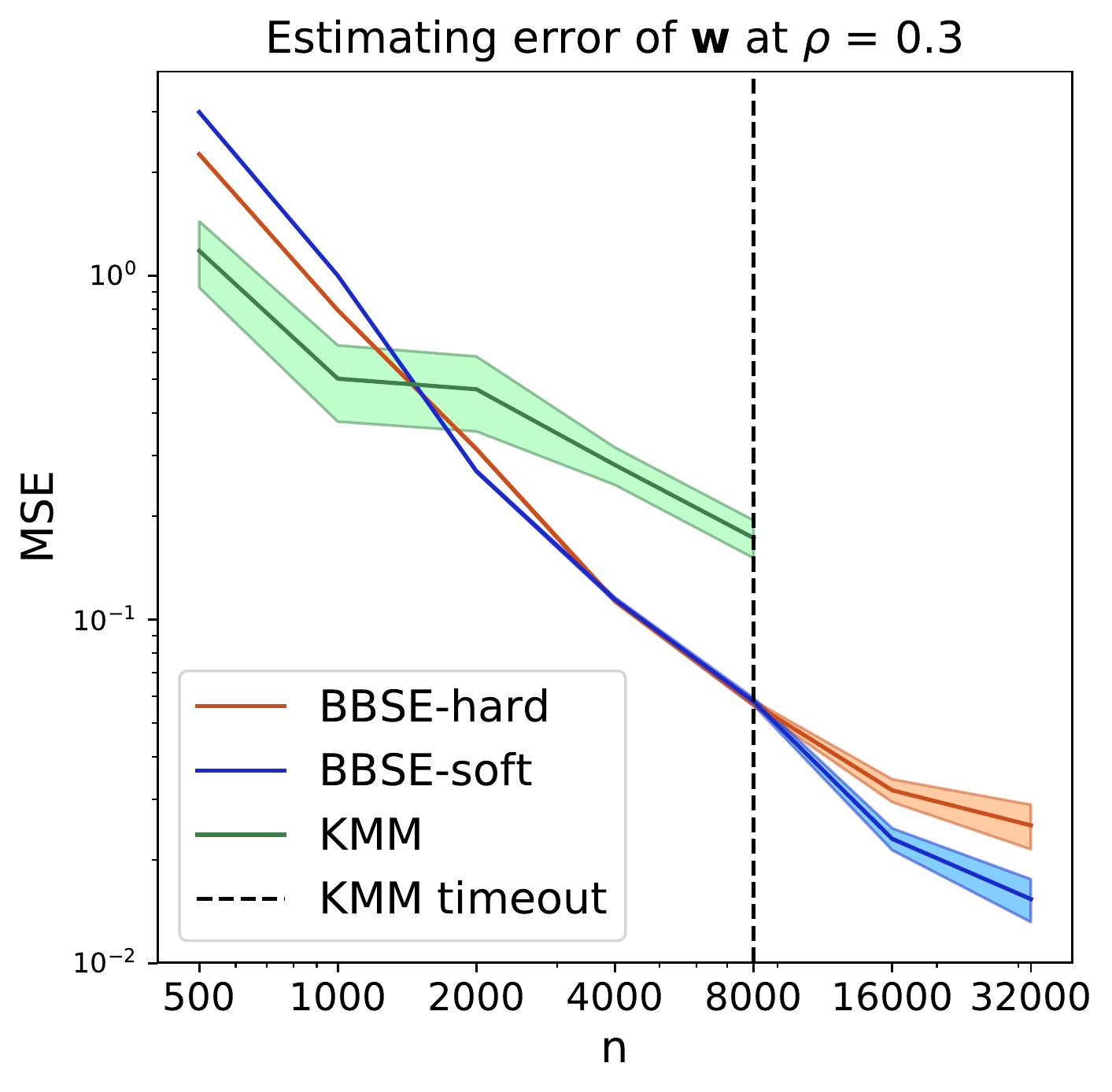}
  \\
  \includegraphics[width=0.32\textwidth]{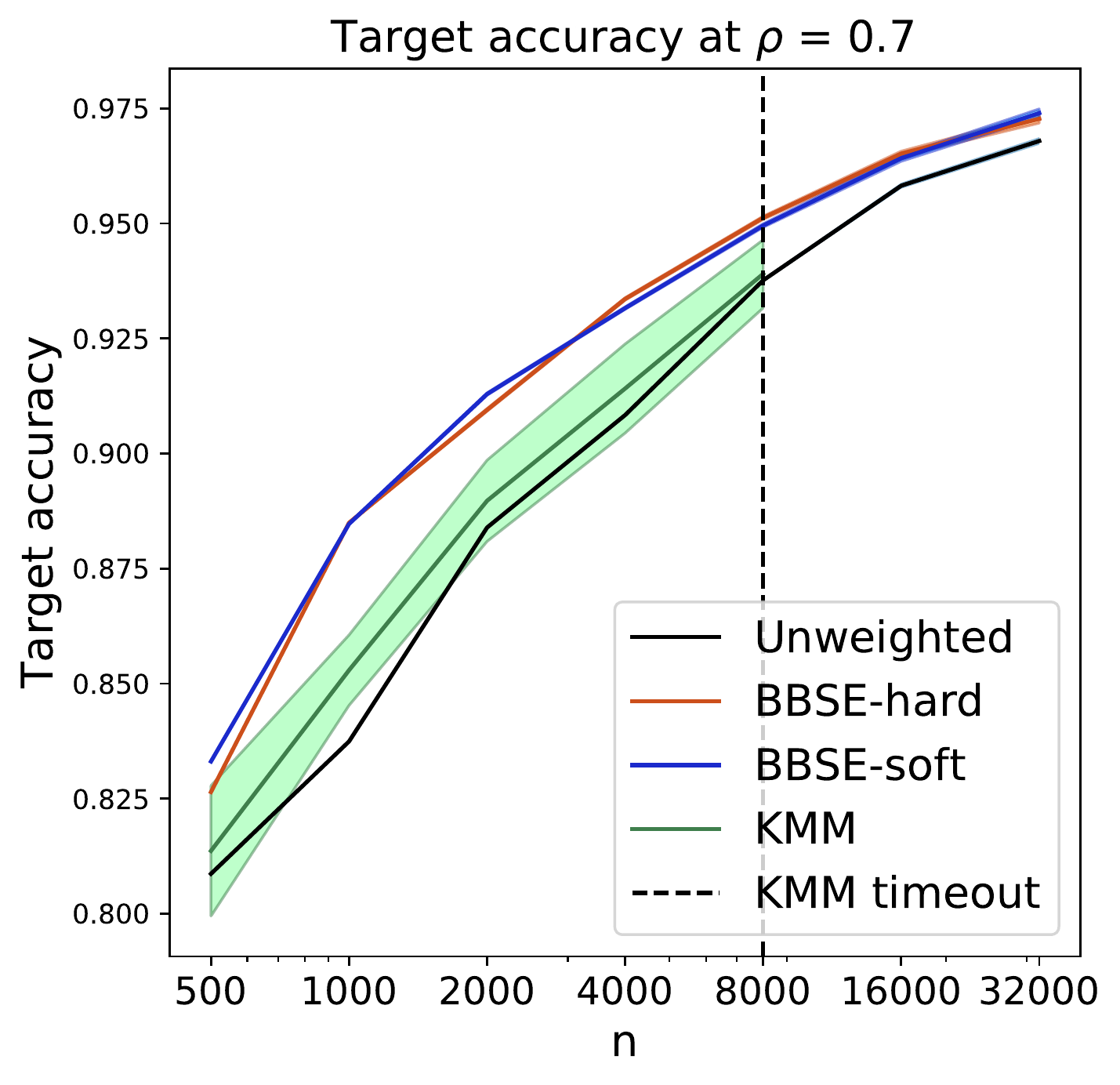}
  \hfill
  \includegraphics[width=0.32\textwidth]{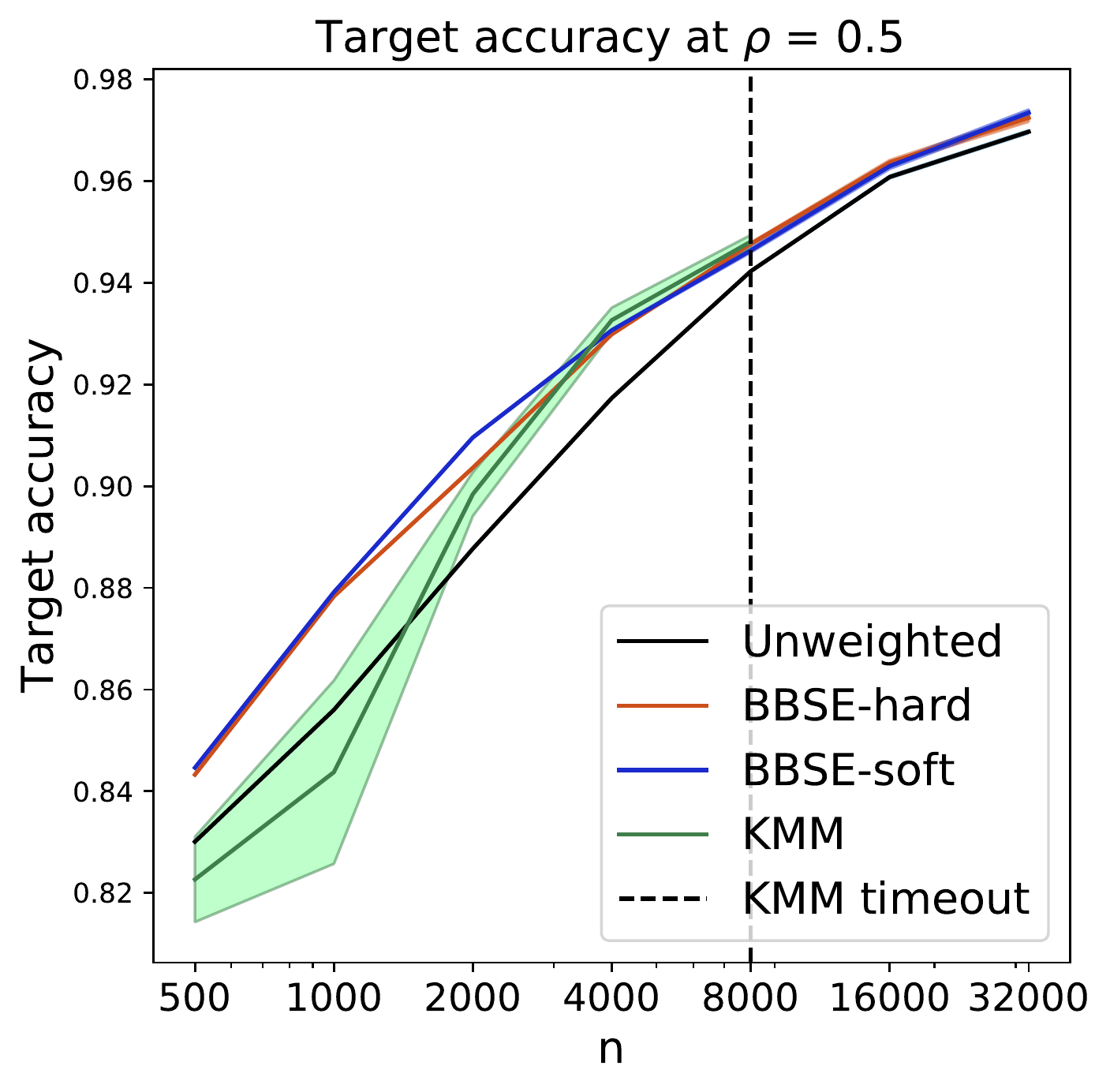}
  \hfill
  \includegraphics[width=0.32\textwidth]{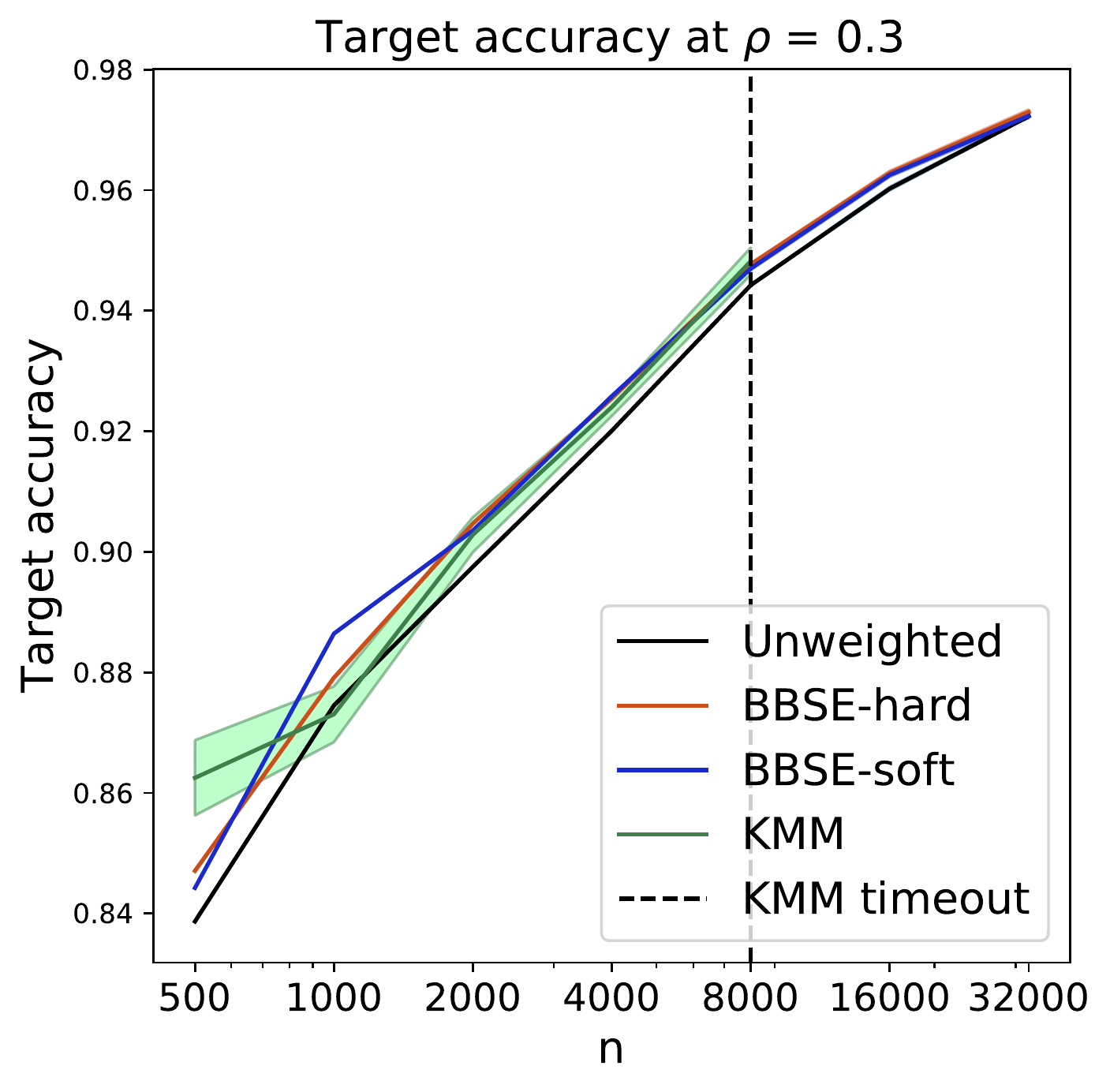}
  \caption{Label-shift estimation and correction on MNIST data with simulated tweak-one shift with parameter $\rho$.
	\label{fig:exp_increase_n_boostQ}}
\end{figure*}

%
\subsection{Black Box Shift Correction (BBSC)}\label{sec:correction}
Our estimator also points to a systematic method
of correcting for label-shift 
via importance-weighted ERM. 
Specifically, we propose the following algorithm:
\begin{algorithm} [h!]                   
	\caption{Domain adaptation via Label Shift}          
	\label{alg:labelshift-alg}                           
	\begin{algorithmic}                    
		\INPUT{ Samples from source distribution $X$, $\mathbf{y}$. Unlabeled data from target distribution $X'$. A class of classifiers $\cF$. Hyperparameter $0<\delta<1/k$.
		}
		\STATE{1. Randomly split the training data into two  $X_1,X_2 \in \R^{n/2\times d}$ and $\vct y_1, \vct y_2 \R^{n/2}$. }
		\STATE{2. Use $X_1,\vct y_1$ to train the classifier and obtain $f\in\cF$.}
		\STATE{3. On the hold-out data set $X_2,\vct y_2$, calculate  the confusion matrix $\hat{\mat C}_{\hat{y},y}$. If ,  }
		\IF{$\sigma_{\min}\left(\hat{\mat C}_{\hat{y},y}\right)\leq \delta$}
		\STATE{Set $\hat{\vct w} = \vct 1$.}
		\ELSE 
		\STATE{Estimate $\hat{\vct w} =  \hat{\mat C}_{\hat{y},y}^{-1} \hat{\vct \mu}_{\hat{y} }$ .}
		\ENDIF
		\STATE{4. Solve the importance weighted  ERM on the $X_1,\vct y_1$ with $\max(\hat{\vct w},\vct 0)$ and obtain $\tilde{f}$.}\ 
		\OUTPUT{ $\tilde{f}$}
	\end{algorithmic}
\end{algorithm}


Note that for classes 
that occur rarely in the test set, 
\algo{} may produce negative importance weights.
During ERM, a flipped sign would cause us 
to \emph{maximize} loss, which is unbounded above. 
Thus, we clip negative weights to $0$. 



%
%
%

Owing to its efficacy and generality, 
our approach can serve as a default tool
to deal with domain adaptation.  
It is one of the first things to try 
even when the label-shift assumption doesn't hold. 
By contrast, the heuristic method of using logistic-regression 
to construct importance weights \citep{bickel2009discriminative}  
lacks theoretical justification 
that the estimated weights are correct. 

Even in the simpler problem of average treatment effect (ATE) estimation,
it's known that using estimated propensity 
can lead to estimators with large variance \citep{kang2007demystifying}. 
The same issue applies in supervised learning.
We may prefer to live with the \emph{biased}  solution from the unweighted ERM 
rather than suffer high variance from 
an unbiased weighted ERM.
Our proposed approach offers 
a consistent low-variance estimator
under label shift.

\vspace{-2px}

\section{Experiments}\label{sec:exp}
\label{sec:experiments}
We experimentally demonstrate 
the power of \algo{} with real data and simulated label shift.  
We organize results into three categories --- 
shift detection \textbf{with BBSD}, 
weight estimation \textbf{with \algo},
and classifier correction \textbf{with BBSC}.
\textbf{BBSE-hard} denotes our method where $f$ yields classifications. In \textbf{BBSE-soft},
$f$ outputs probabilities.

\textbf{Label Shift Simulation }
To simulate distribution shift in our experiments, 
we adopt the following protocols:
First, we split the original data into train, validation, and test sets.
Then, given distributions $p(y)$ and $q(y)$,
we generate each set by sampling with replacement from the appropriate split.
In \textbf{knock-out shift}, we knock out a fraction $\delta$ of data points from a given class from training and validation sets.
In \textbf{tweak-one shift}, we assign a probability $\rho$ to one of the classes, the rest of the mass is spread evenly among the other classes. 
In \textbf{Dirichlet shift}, 
we draw $p(y)$ from a Dirichlet distribution with concentration parameter $\alpha$. 
With uniform $p(y)$, 
Dirichlet shift is bigger for smaller $\alpha$. 

\textbf{Label-shift detection } 
We conduct nonparametric two-sample tests 
as described in Section~\ref{sec:detection} 
using the MNIST handwritten digits data set. 
To simulate the label-shift, 
we randomly split the training data 
into a training set, a validating set 
and a test set, each with 20,000 data points, 
and apply knock-out shift on class $y=5$~\footnote{Random choice for illustration, method works on all classes.}. 
Note that $p(y)$ and $q(y)$
differ increasingly as $\delta$ grows large,
making shift detection easier. 
We obtain $f$ by training 
a two-layer ReLU-activated Multilayer Perceptron (MLP) 
with 256 neurons on the training set for five epochs. 
We conduct a two-sample test of
whether the distribution of 
$f(\text{Validation Set})$ and $f(\text{Test Set})$ 
are the same using the Kolmogorov-Smirnov test.  
The results, summarized in Figure~\ref{fig:exp_hypothesis_testing}, 
demonstrate that BBSD 
(1) produces a $p$-value that distributes uniformly 
when $\delta=0$ 
\footnote{Thus we can 
control Type I error at any 
significance level.}
 (2) provides more power (less Type II error) 
 than the state-of-the-art kernel two-sample test 
 that discriminates $p(\vct x)$ and $q(\vct x)$ 
 at $\delta=0.5$, 
 and (3) gets better 
 as we train the black-box predictor even more.

\textbf{Weight estimation and label-shift correction } 
We evaluate \algo{} 
on MNIST by simulating label shift 
and datasets of various sizes. 
Specifically, 
we split the training data set randomly in two,
using first half to train $f$ 
and the second half to estimate $\vct w$. 
We use then use the full training set for weighted ERM. 
As before, $f$ is a two-layer MLP.
For fair comparisons with baselines, 
the full training data set is used throughout 
(since they do not need $f$ 
without data splitting). 
We evaluate our estimator $\hat{\vct w}$ 
against the ground truth $\vct w$ 
and by the prediction accuracy of BBSC on the test set.
To cover a variety of different types of label-shift,
we take $p(y)$ as a uniform distribution and generate $q(y)$ with \emph{Dirichlet shift} 
for $\alpha=0.1, 1.0, 10.0$ (Figure~\ref{fig:exp_increase_n_dirichlet}).

\begin{figure}[tb]
	\centering
  	\includegraphics[width=.9\columnwidth]{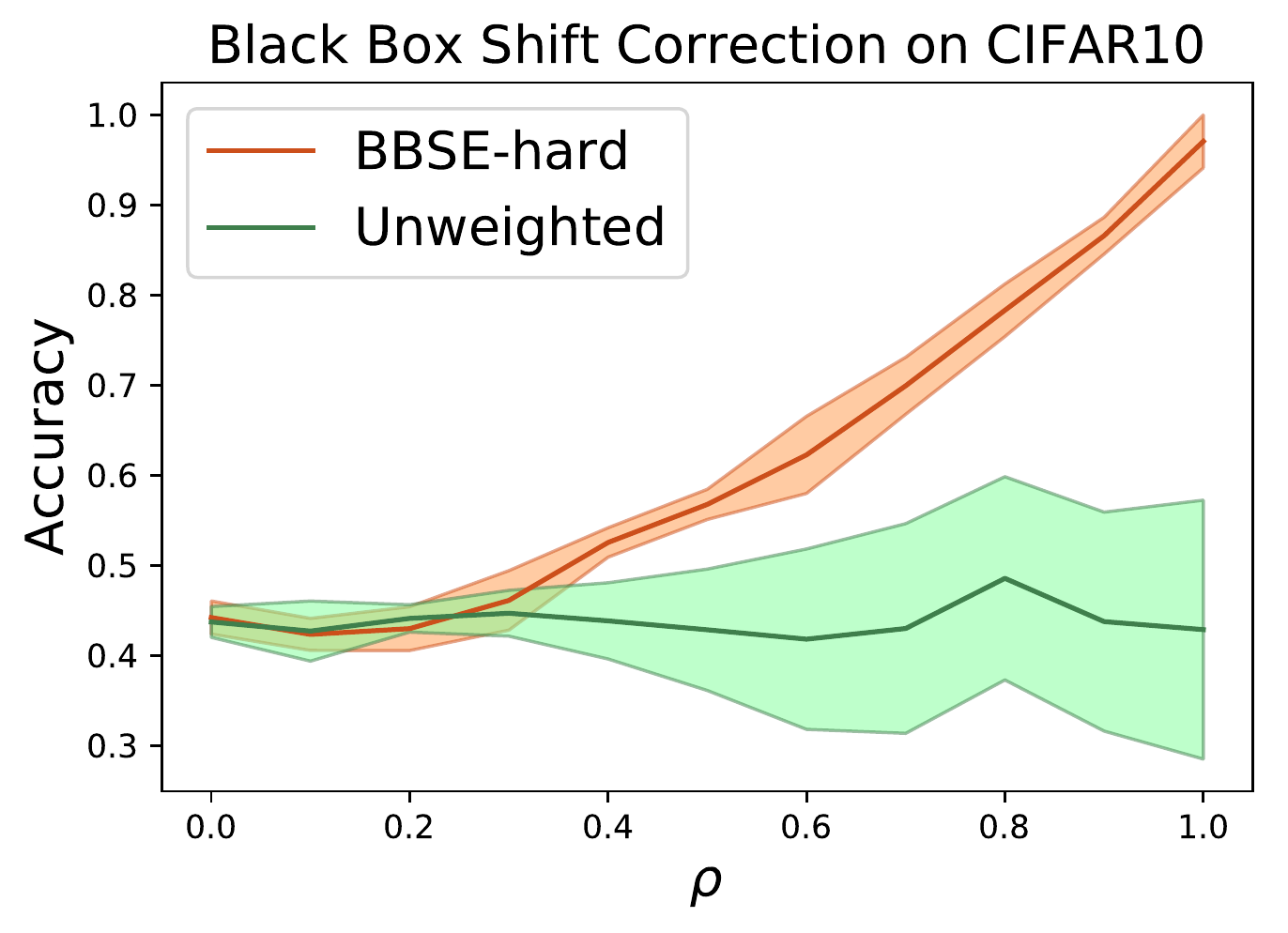}
  	\includegraphics[width=.9\columnwidth]{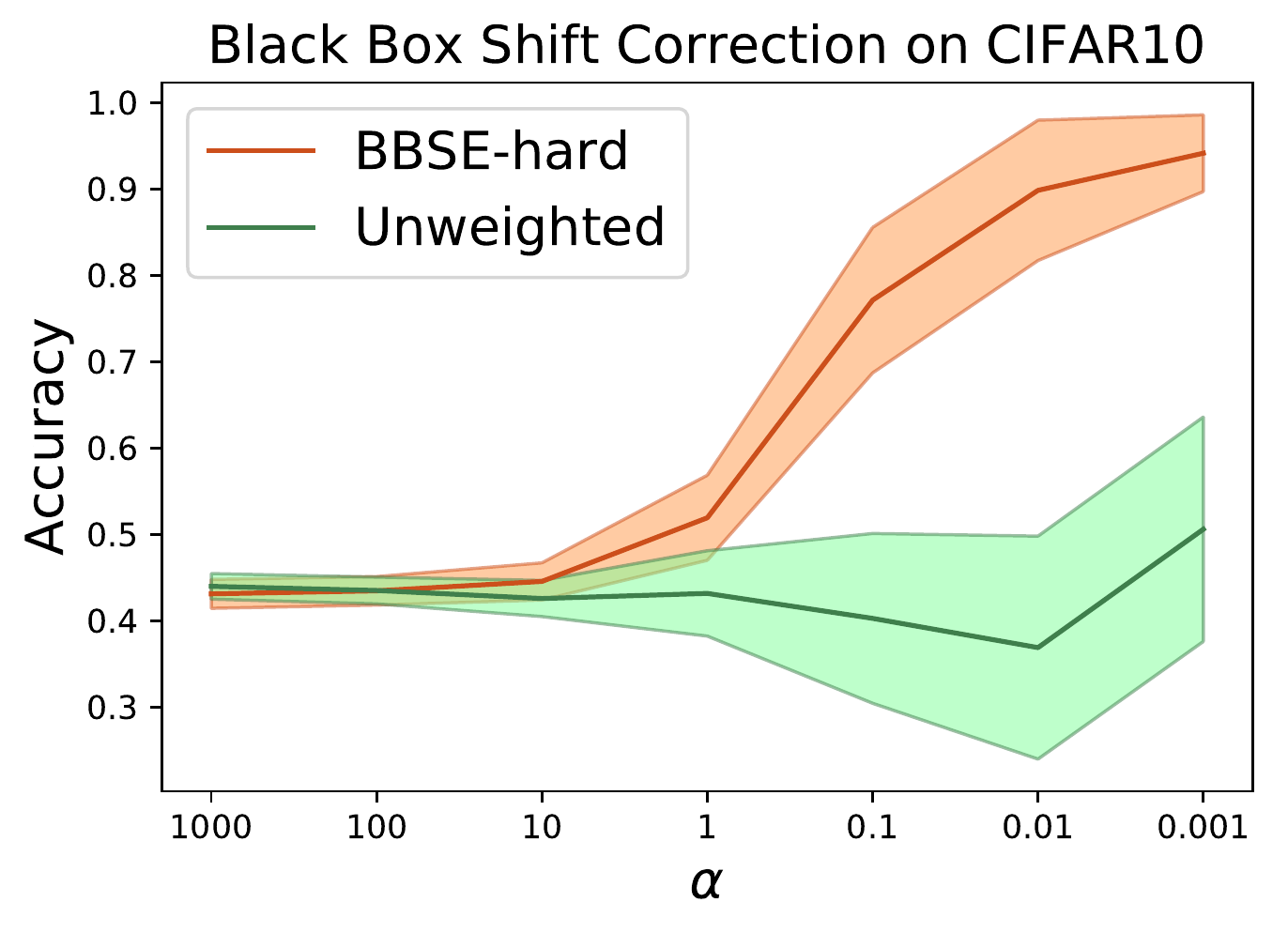}
  	\caption{Accuracy of BBSC on CIFAR 10 with (top) tweak-one shift and
    (bottom) Dirichlet shift. 
    \label{fig:cifar10}}
    \vspace{-10px}
\end{figure}

\textbf{Label-shift correction for CIFAR10 } 
Next, we extend our experiments to the CIFAR dataset,
using the same MLP and this time allowing it 
to train for 10 epochs.
We consider both tweak-one and Dirichlet shift,
and compare \algo{} to the unweighted classifier under varying degrees of shift (Figure \ref{fig:cifar10}).
For the tweak-one experiment, 
we try $\rho \in \{0.0, 0.1, ..., 1.0\}$,
averaging results over all 10 choices 
of the tweaked label, and plotting the variance.
For the Dirichlet experiments, 
we sample $20$ $q(y)$ for every choice of $\alpha$ 
in the range \{1000, 100, ..., .001\}.
Because kernel-based baselines cannot handle datasets this large or high-dimensional,  
we compare only to unweighted ERM.

\textbf{Kernel mean matching (KMM) baselines }
We compare \algo{} to the state-of-the-art 
kernel mean matching (KMM) methods. 
For the detection experiments (Figure~\ref{fig:exp_hypothesis_testing}), 
our baseline is the kernel B-test~\citep{zaremba2013b}, 
an extension of the kernel max mean discrepancy (MMD) 
test due to \citet{gretton2012kernel} 
that boasts nearly linear-time computation 
and little loss in power.
We compare \algo{} to a KMM approach \citet{zhang2013domain}, that solves
$$
\min_{\vct w} \|  \mat C_{\vct x| y} (\vct \nu_{y}\circ \vct w) -  \vct \mu_{\vct x}\|_{\cH}^2,
$$
where we use operator 
$\mat C_{\vct x| y} := \E[\phi(\vct x) | \psi(y)]$ 
and function $\vct \mu_{x} := \E_Q[\phi(\vct x)]$ 
to denote the kernel embedding 
of $p(\vct x|y)$ and $p_Q(x)$ respectively. 
Note that under the label-shift assumption, 
$\mat C_{\vct x| y}$ is the same for $P$ and $Q$. 
Also note that since $\cY$ is discrete, 
$\psi(y)$ is simply the one-hot representation of $y$, 
so $\vct \nu_{y}$ is the same as our definition before 
and $\mat C_{\vct x| y}$, 
$\vct \nu_{y}$ and $\vct \mu_{\vct x}$ 
must be estimated from finite data.
The proposal involves a constrained optimization 
by solving a Gaussian process regression 
with automatic hyperparameter choices 
through marginal likelihood. 

For fair comparison, 
we used the original authors' implementations 
as baselines \footnote{\url{https://github.com/wojzaremba/btest}, \url{http://people.tuebingen.mpg.de/kzhang/Code-TarS.zip}} and also used the \emph{median trick} 
to adaptively tune the RBF kernel's hyperparameter. 
\textbf{A key difference} is that \algo{}
matches the distribution of $\hat{y}$ rather than distribution of $\vct x$ like \citep{zhang2013domain} 
and we learn $f$ through supervised learning 
rather than by specifying a feature map 
$\phi$ by choosing a kernel up front. 

Note that KMM, like many kernel methods,
requires the construction and inversion 
of an $n\times n$ Gram matrix, 
which has complexity of $O(n^3)$. 
This hinders its application 
to real-life machine learning problems 
where $n$ will often be 100s of thousands. 
In our experiments, 
we find that the largest $n$ 
for which we can feasibly run the KMM code 
is roughly $8,000$ 
and that is where we unfortunately 
have to stop for the MNIST experiment. 
For the same reason, we cannot run KMM for the CIFAR10 experiments. 
The MSE curves in Figure~\ref{fig:exp_increase_n_dirichlet} 
for estimating $\vct w$ suggest 
that the convergence rate of KMM 
is slower than \algo{} 
by a polynomial factor and 
that \algo{} better handles large datasets.

\vspace{-2px}

\section{Discussion}
\label{sec:discussion}


\textbf{Constructing the training Set }
The error bounds on our estimates 
depend on the norm of the true vector $w(y):= q(y)/p(y)$. This confirms the common sense that absent any assumption on $q(y)$,
and given the ability to select class-conditioned examples for annotations
one should build a dataset 
with uniform $p(y)$.
Then it's always possible to apply \algo{}
successfully at test time to correct $f$.


\textbf{Sporadic Shift }
In some settings, $p(y)$
might change only sporadically.
In these cases, when no label shift occurs, applying BBSC might damage the classifier.
For these cases, 
we prose to combine detection and estimation,
correcting the classifier only when a shift 
has likely occurred.

\textbf{Using known predictor }
In our experiments, $f$ has been trained using a random split of the data set, which makes \algo to perform worse than baseline when the data set is extremely small.  In practice, especially in the context of web services, there could be a natural predictor $f$ that is currently being deployed whose training data were legacy and have little to do with the two distributions that we are trying to distinguish. In this case, we do not lose that factor of $2$ and we do not suffer from the variance in training $f$ with a small amount of data. This could allow us to detect mild shift in distributions in very short period of time. Making it suitable for applications such as financial market prediction. 

\textbf{\algo{} with degenerate confusion matrices }
In practice, sometime confusion matrices will be degenerate. 
For instance, when a class $i$ is rare under $P$, 
and the features are only partially predictive, 
we might find that 
$p(f(\vct x)=i) = 0$.
In these cases, two straightforward variations on the black box method may still work:
First, while our analysis focuses on confusion matrices, it easily extends to any operator $f$,
such as soft probabilities. 
If each class $i$, 
even if $i$ is never the argmax for any example,
so long as $p(\hat{y}=i |y=i) > p(\hat{y}=i|y=j)$
for any $j \neq i$, 
the soft confusion matrix will be invertible.
Even when we produce and operator with an invertible confusion matrix, 
two options remain:
We can merge $c$ classes together, yielding a $(k-c) \times (k-c)$ 
invertible confusion matrix.
While we might not be able to estimate the frequencies 
of those $c$ classes,
we can estimate the others accurately. 
Another possibility is to compute the pseudo-inverse.

\textbf{Future Work }
As a next step, we plan to extend our methodology to the streaming setting.
In practice, label distributions
tend to shift progressively, 
presenting a new challenge: 
if we apply \algo{}
on trailing windows,
then we face a trade-off.
Looking far back increases $m$,
lowering estimation error,
but the estimate will be less fresh. 
The use of propensity weights $w$ on $y$ 
makes \algo{} amenable to doubly-robust estimates, 
the typical bias-variance tradeoff, 
and related techniques, 
common in covariate shift correction.

\vspace{-2px}

\section*{Acknowledgments}
\label{sec:acknowledgments}
We are grateful for extensive insightful discussions and feedback from Kamyar Azizzadenesheli, Kun Zhang, Arthur Gretton, Ashish Khetan Kumar, Anima Anandkumar, Julian McAuley, Dustin Tran, Charles Elkan, Max G'Sell, Alex Dimakis, Gary Marcus, and Todd Gureckis.


\bibliographystyle{apa-good}
\bibliography{labelshift}

\begin{thebibliography}{24}
\expandafter\ifx\csname natexlab\endcsname\relax\def\natexlab#1{#1}\fi
\expandafter\ifx\csname url\endcsname\relax
  \def\url#1{{\tt #1}}\fi
\expandafter\ifx\csname urlprefix\endcsname\relax\def\urlprefix{URL }\fi

\bibitem[{Bickel et~al.(2009)Bickel, Br{\"u}ckner, \&
  Scheffer}]{bickel2009discriminative}
Bickel, S., Br{\"u}ckner, M., \& Scheffer, T. (2009).
\newblock Discriminative learning under covariate shift.
\newblock {\em Journal of Machine Learning Research\/}, {\em 10\/}(Sep),
  2137--2155.

\bibitem[{Bishop(1995)}]{bishop1995neural}
Bishop, C.~M. (1995).
\newblock {\em Neural networks for pattern recognition\/}.
\newblock Oxford university press.

\bibitem[{Buck et~al.(1966)Buck, Gart et~al.}]{buck1966comparison}
Buck, A., Gart, J., et~al. (1966).
\newblock Comparison of a screening test and a reference test in epidemiologic
  studies. ii. a probabilistic model for the comparison of diagnostic tests.
\newblock {\em American Journal of Epidemiology\/}.

\bibitem[{Chan \& Ng(2005)}]{chan2005word}
Chan, Y.~S., \& Ng, H.~T. (2005).
\newblock Word sense disambiguation with distribution estimation.
\newblock In {\em Proceedings of the 19th international joint conference on
  Artificial intelligence\/}, (pp. 1010--1015). Morgan Kaufmann Publishers Inc.

\bibitem[{Deng et~al.(2009)Deng, Dong, Socher, Li, Li, \&
  Fei-Fei}]{deng2009imagenet}
Deng, J., Dong, W., Socher, R., Li, L.-J., Li, K., \& Fei-Fei, L. (2009).
\newblock Imagenet: A large-scale hierarchical image database.
\newblock In {\em CVPR\/}.

\bibitem[{Elkan(2001)}]{elkan2001foundations}
Elkan, C. (2001).
\newblock The foundations of cost-sensitive learning.
\newblock In {\em IJCAI\/}.

\bibitem[{Forman(2008)}]{forman2008quantifying}
Forman, G. (2008).
\newblock Quantifying counts and costs via classification.
\newblock {\em Data Mining and Knowledge Discovery\/}.

\bibitem[{Gretton et~al.(2012)Gretton, Borgwardt, Rasch, Sch{\"o}lkopf, \&
  Smola}]{gretton2012kernel}
Gretton, A., Borgwardt, K.~M., Rasch, M.~J., Sch{\"o}lkopf, B., \& Smola, A.
  (2012).
\newblock A kernel two-sample test.
\newblock {\em Journal of Machine Learning Research\/}, {\em 13\/}(Mar),
  723--773.

\bibitem[{Gretton et~al.(2009)Gretton, Smola, Huang, Schmittfull, Borgwardt, \&
  Sch{\"o}lkopf}]{gretton2009covariate}
Gretton, A., Smola, A.~J., Huang, J., Schmittfull, M., Borgwardt, K.~M., \&
  Sch{\"o}lkopf, B. (2009).
\newblock Covariate shift by kernel mean matching.
\newblock {\em Journal of Machine Learning Research\/}.

\bibitem[{Heckman(1977)}]{heckman1977sample}
Heckman, J.~J. (1977).
\newblock Sample selection bias as a specification error (with an application
  to the estimation of labor supply functions).

\bibitem[{Huang et~al.(2007)Huang, Gretton, Borgwardt, Sch{\"o}lkopf, \&
  Smola}]{huang2007correcting}
Huang, J., Gretton, A., Borgwardt, K.~M., Sch{\"o}lkopf, B., \& Smola, A.~J.
  (2007).
\newblock Correcting sample selection bias by unlabeled data.
\newblock In {\em Advances in neural information processing systems\/}.

\bibitem[{Kang \& Schafer(2007)}]{kang2007demystifying}
Kang, J.~D., \& Schafer, J.~L. (2007).
\newblock Demystifying double robustness: A comparison of alternative
  strategies for estimating a population mean from incomplete data.
\newblock {\em Statistical science\/}, {\em 22\/}(4), 523--539.

\bibitem[{Manski \& Lerman(1977)}]{manski1977estimation}
Manski, C.~F., \& Lerman, S.~R. (1977).
\newblock The estimation of choice probabilities from choice based samples.
\newblock {\em Econometrica: Journal of the Econometric Society\/}.

\bibitem[{Ramdas et~al.(2015)Ramdas, Reddi, P{\'o}czos, Singh, \&
  Wasserman}]{ramdas2015decreasing}
Ramdas, A., Reddi, S.~J., P{\'o}czos, B., Singh, A., \& Wasserman, L.~A.
  (2015).
\newblock On the decreasing power of kernel and distance based nonparametric
  hypothesis tests in high dimensions.
\newblock In {\em AAAI\/}, (pp. 3571--3577).

\bibitem[{Rosenbaum \& Rubin(1983)}]{rosenbaum1983central}
Rosenbaum, P.~R., \& Rubin, D.~B. (1983).
\newblock The central role of the propensity score in observational studies for
  causal effects.
\newblock {\em Biometrika\/}, {\em 70\/}(1), 41--55.

\bibitem[{Saerens et~al.(2002)Saerens, Latinne, \&
  Decaestecker}]{saerens2002adjusting}
Saerens, M., Latinne, P., \& Decaestecker, C. (2002).
\newblock Adjusting the outputs of a classifier to new a priori probabilities:
  a simple procedure.
\newblock {\em Neural computation\/}, {\em 14\/}(1), 21--41.

\bibitem[{Sch{\"o}lkopf et~al.(2012)Sch{\"o}lkopf, Janzing, Peters, Sgouritsa,
  Zhang, \& Mooij}]{scholkopf2012causal}
Sch{\"o}lkopf, B., Janzing, D., Peters, J., Sgouritsa, E., Zhang, K., \& Mooij,
  J. (2012).
\newblock On causal and anticausal learning.
\newblock In {\em International Coference on International Conference on
  Machine Learning (ICML-12)\/}, (pp. 459--466). Omnipress.

\bibitem[{Shimodaira(2000)}]{shimodaira2000improving}
Shimodaira, H. (2000).
\newblock Improving predictive inference under covariate shift by weighting the
  log-likelihood function.
\newblock {\em Journal of statistical planning and inference\/}.

\bibitem[{Storkey(2009)}]{storkey2009training}
Storkey, A. (2009).
\newblock When training and test sets are different: characterizing learning
  transfer.
\newblock {\em Dataset shift in machine learning\/}.

\bibitem[{Sugiyama et~al.(2008)Sugiyama, Nakajima, Kashima, Buenau, \&
  Kawanabe}]{sugiyama2008direct}
Sugiyama, M., Nakajima, S., Kashima, H., Buenau, P.~V., \& Kawanabe, M. (2008).
\newblock Direct importance estimation with model selection and its application
  to covariate shift adaptation.
\newblock In {\em Advances in neural information processing systems\/}.

\bibitem[{Zadrozny(2004)}]{zadrozny2004learning}
Zadrozny, B. (2004).
\newblock Learning and evaluating classifiers under sample selection bias.
\newblock In {\em Proceedings of the twenty-first international conference on
  Machine learning\/}, (p. 114). ACM.

\bibitem[{Zaremba et~al.(2013)Zaremba, Gretton, \& Blaschko}]{zaremba2013b}
Zaremba, W., Gretton, A., \& Blaschko, M. (2013).
\newblock B-test: A non-parametric, low variance kernel two-sample test.
\newblock In {\em Advances in neural information processing systems\/}, (pp.
  755--763).

\bibitem[{Zhang et~al.(2013)Zhang, Sch{\"o}lkopf, Muandet, \&
  Wang}]{zhang2013domain}
Zhang, K., Sch{\"o}lkopf, B., Muandet, K., \& Wang, Z. (2013).
\newblock Domain adaptation under target and conditional shift.
\newblock In {\em International Conference on Machine Learning\/}, (pp.
  819--827).

\bibitem[{Zhu et~al.(2010)Zhu, Gibson, Jun, Rogers, Harrison, \&
  Kalish}]{zhu2010cognitive}
Zhu, X., Gibson, B.~R., Jun, K.-S., Rogers, T.~T., Harrison, J., \& Kalish, C.
  (2010).
\newblock Cognitive models of test-item effects in human category learning.
\newblock In {\em ICML\/}.

\end{thebibliography}

\appendix
\newpage

\section{Additional discussion}
In this section we provide a few answers to some questions people may have when using our proposed techniques.

\paragraph{What if the label-shift assumption does not hold?}
In many applications, we do not know whether label-shift is a reasonable assumption or not. In particular, whenever there are unobserved variables that affects both $\vct x$ and $y$, then neither label-shift nor covariate-shift is true.  However, label shift could still be a good approximation in the in the finite sample environment.
Luckily, we can test whether the label-shift assumption is a good approximation in a data-driven fashion via the kernel two-sample tests. In particular, let $\phi: \cX \rightarrow \cF$ be an arbitrary feature map that (possibly reduces the dimension of $\vct x$) and $k: \cF\times \cF \rightarrow \R$ be the kernel function that induces a RKHS $\cH$. Let $\vct w  =  [q(y)/ p(y)]_{y=1,...,k}$, then
$$
\E_{p} \left[  \vct w(y)  k(\phi( \vct x),\cdot)\right]   =    \E_{q}  \left[ k(\phi( \vct x),\cdot) \right].
$$
The LHS can be estimated by plugging in $\hat{\vct w}$ and a stochastic approximation of the expectation using labeled data from the source domain and the RHS can be estimated by the sample mean using unlabeled data from the target domain.  In particular, if label-shift assumption is true or a good approximation, then
$$\| \frac{1}{n}\sum_{i=1}^n\left[  \hat{\vct w}(y_i)  k(\phi( \vct x_i),\cdot)\right] -  \frac{1}{m}\sum_{j=1}^m k(\phi( \vct x'_j),\cdot)  \|_\cH^2$$
should be on the same order as the statistical error that we can calculate by $m,n$ and the error of $\hat{\vct w}$ in estimating  $\vct w$. 

\paragraph{Model selection criterion and the choice of $f$.}
Our analysis assumes that $f$ is fixed and given, but in practice, often we need to train $f$ from the same data set. Given a number of choices, one may wonder which blackbox predictor $f$ should we prefer out of a collection of $\cF$?
Our theoretical results suggest a natural quantity: the smallest singular value of the confusion matrix, for choosing the blackbox predictors. Note that the smallest singular value is a quantity that can be estimated using only labeled data from the source domain. Therefore a practical heuristic to use is to the $f$ that maximizes the smallest singular value of the corresponding $\hat{C}_f$.  Figure~\ref{fig:smallestsigma} plots the smallest singular value of the confusion matrices as the number of epochs of training $f$ gets larger. The model we use is the same multi-layer perceptron that we used for our experiments and the source distribution is one that we knocks off 80\% of the fifth class. This is the same model and data set we used in  Figure~\ref{fig:power_vs_oracle}. Referring to $\delta=0.8$ in Figure~\ref{fig:power_vs_oracle}, we see that the test power of $f$ that is trained for only one epoch is much lower than the $f$ that is trained for five epochs, and the gap in the smallest singular values is predicative of the fact at least qualitatively.
\begin{figure}[h!]
	\centering
	\includegraphics[width=.9\columnwidth]{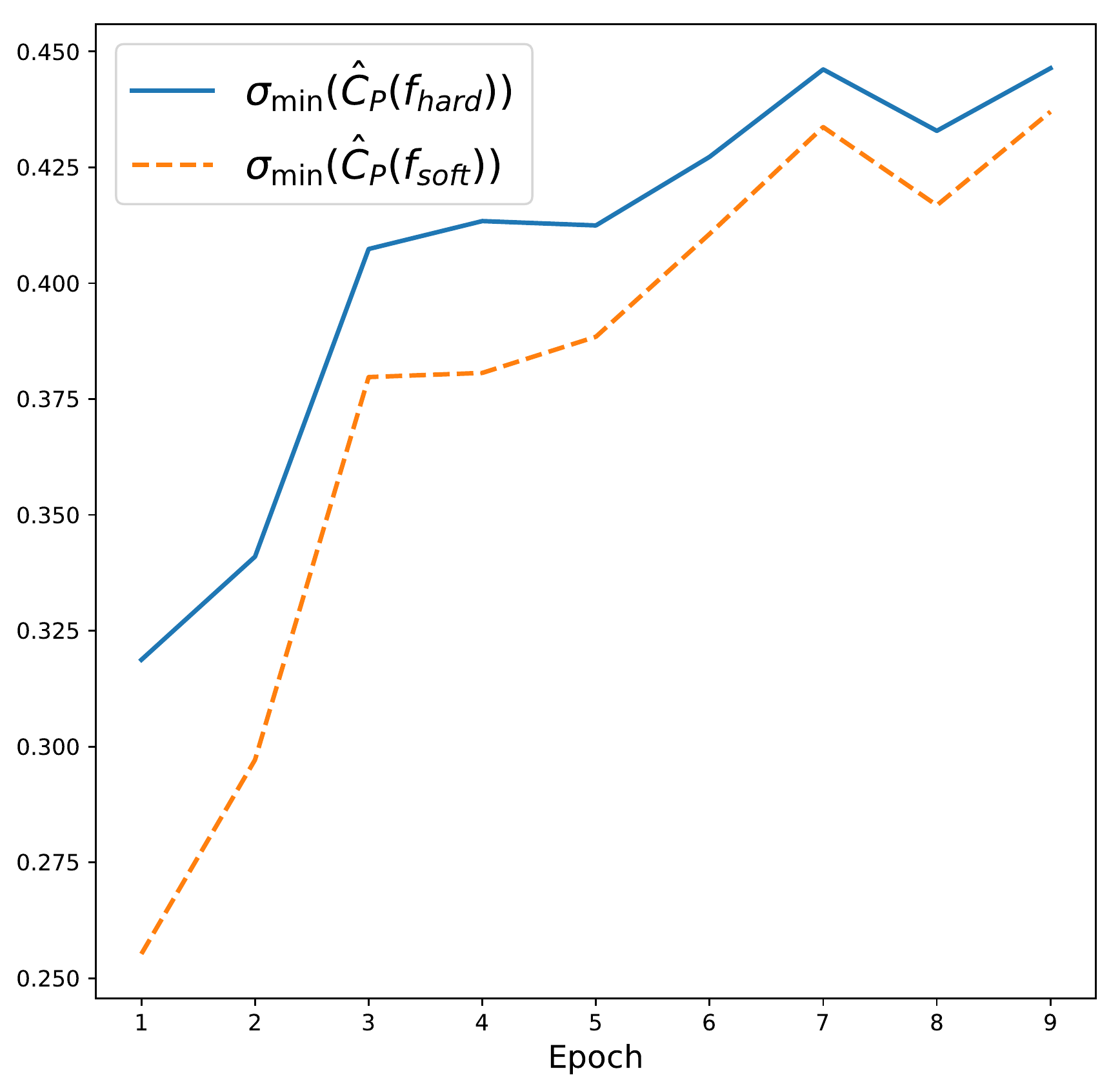}
	\caption{The smallest singular value of the estimated confusion matrix: $\hat{C}_f$ under distribution $p$ as a function of the number of epochs we train the classifiers on.
		\label{fig:smallestsigma}}
\end{figure}

\paragraph{Is data splitting needed?}
Recall that we train the model $f$ and estimate $\vct w$ using two independent splits of the labeled data set drawn from the same distribution. In practice, especially when $n$ is large, using the same data to train $f$ and to estimate $\vct w$ will be more data efficient. This comes at a price of a small bias. It is unclear how to quantify that bias but the data-reuse version could be useful in practice as a heuristic.

\section{Proofs}\label{app:proofs}
We present the proofs of Lemma~\ref{lem:matching_confusion} and Proposition~\ref{prop:consistency} in this Appendix.

\begin{proof}[Proof of Lemma~\ref{lem:matching_confusion}]
	By the law of total probability
\begin{align*}
	&q(\hat{y} | y)  = \sum_{y\in\cY}  q(\hat{y} | \vct x, y) q(\vct x|y)  =  \sum_{y\in\cY}  q(\hat{y} | \vct x, y) p(\vct x|y)  \\ 
	&=  \sum_{y\in\cY}  p_{f}(\hat{y} | \vct x)  p(\vct x|y)   = \sum_{y\in\cY}  p(\hat{y} | \vct x, y) p(\vct x|y) = p(\hat{y} | y).
	\end{align*}
We applied A.1 to the second equality, and used the conditional independence $\hat{y}\independent y  |  x$ under 
$P$ and $Q$ together with $p(\hat{y} | \vct x)$ being 
determined by $f$, which is fixed. 
\vspace{-15px}
\end{proof}

\bigskip

\begin{proof}[Proof of Proposition~\ref{prop:consistency}]
A.2 ensures 
that $\vct w<\infty$.	 
By Assumption A.3, ${\mat C}_{\hat{y},y}$ is invertible. Let $\delta>0$ be its smallest singular value. We bound the probability that 
$\hat{\mat C}_{\hat{y},y}$ is not invertible:

\begin{align*}
&\P(\hat{\mat C}_{\hat{y},y}\text{ is not invertible}) \leq  \P(\sigma_{\min}(\hat{\mat C}_{\hat{y},y}) < \delta/2) \\
& \leq \P(\|\hat{\mat C}_{\hat{y},y}-\mat C_{\hat{y},y}\|_2 \geq \delta/2 ) \explain{\leq}{\text{pigeon hole}} \P(\|\hat{\mat C}_{\hat{y},y}-\mat C_{\hat{y},y}\|_F \geq \frac{\delta}{2\sqrt{k}}) \\
&\explain{\leq}{\text{pigeon hole}} \P(\exists (i,j)\in[k]^2, \text{ s.t.} |[\hat{\mat C}_{\hat{y},y}]_{i,j}- [\mat C_{\hat{y},y}]_{i,j}| \geq \frac{\delta}{2k^{1.5}}) \explain{\leq}{\text{Hoeffding}} 2e^{-\frac{n\delta^2}{4k^3}}.
\end{align*}
By the convergence of geometric series
$
\sum_{n}\P(\hat{\mat C}_{\hat{y},y}\text{ is not invertible})  < +\infty$.
This allows us to invoke the First Borel-Cantelli Lemma, which shows
\begin{equation}\label{eq:almostsure_invertible}
\P(\hat{\mat C}_{\hat{y},y}\text{ is not invertible} \text{ i.o.})  = 0.
\end{equation}
This ensures that as $n\rightarrow \infty$, $\hat{\mat C}_{\hat{y},y}$ is invertible almost surely.
By the strong law of large numbers (SLLN), as $n\rightarrow \infty$
$\hat{\mat C}_{\hat{y},y}\overset{\text{a.s.}}{\longrightarrow} \mat C_{\hat{y},y}$
and
$\hat{\vct \nu}_{y} \overset{\text{a.s.}}{\longrightarrow} \vct \nu_{y}$.
Similarly, as $m\rightarrow \infty$, 
$\hat{\vct \mu}_{\hat{y}}  \overset{\text{a.s.}}{\longrightarrow}  \vct \mu_{\hat{y}}$.
Combining these with \eqref{eq:almostsure_invertible} and applying the continuous mapping theorem with the fact that the 
inverse of an invertible matrix is a continuous mapping we get that 
$$
\hat{\vct w} =  [\hat{\mat C}_{\hat{y},y}]^{-1} \hat{\vct \mu}_{\hat{y}}   \overset{\text{a.s.}}{\longrightarrow}  \vct w, \text{ and }
\hat{\vct \mu}_{y} =  \diag(\hat{\vct \nu}_y) \hat{\vct w} \overset{\text{a.s.}}{\longrightarrow}   \vct \mu_{y}.
$$
\end{proof}

\section{Concentration inequalities}
\label{sec:technical-lemmas}
\begin{lemma}[Hoeffding's inequality]
Let $x_1,...,x_n$ be independent random variables bounded by $[a_i,b_i]$.
Then $\bar{x} = \frac{1}{n}\sum_{i=1}^n x_i$ obeys for any $t>0$
$$
\P(\left|\bar{x} - \E[\bar{x}]\right|\geq t) \leq 2\exp\left(-\frac{2n^2t^2}{\sum_{i=1}^n(b_i-a_i)^2}\right).
$$
\end{lemma}
\begin{lemma}[Matrix Bernstein Inequality (rectangular case)]\label{lem:matrixbernstein}
	Let $\mat Z_1,...,\mat Z_n$ be independent random matrices with dimension $d_1\times d_2$ and each satisfy 
	$$\E \mat Z_i  =  \mat 0\text{ and }  \|\mat Z_i\| \leq R$$
	almost surely. Define the variance parameter
	$$
	\sigma^2 - \max\{ \|\sum_i \E[\mat Z_i  \mat Z_i^T]\|, \|\sum_i\E[\mat Z_i^T \mat Z_i]\|  \}.
	$$
	Then for all $t\geq 0$,
	$$
	\P\left(  \|\sum_i \mat Z_i\| \geq t\right) \leq (d_1 + d_2) \cdot e^{\frac{-t^2}{\sigma^2 + Rt/3}}.
	$$
\end{lemma}


\end{document}